\documentclass{article}

\usepackage[utf8]{inputenc} 
\usepackage[T1]{fontenc}    
\usepackage{pifont}
\usepackage{fullpage}
\usepackage{multicol,lipsum,xcolor}
\usepackage{hyperref}       
\usepackage{url}            
\usepackage{amsmath,amsthm,amssymb,bm}
\usepackage{bbold}
\DeclareMathOperator{\argmin}{argmin}  
\DeclareMathOperator{\argmax}{argmax}
\usepackage{algorithm,algorithmic}
\usepackage{caption}
\usepackage{graphicx}
\usepackage{enumerate}
\usepackage{enumitem}
\usepackage{natbib}

\newtheorem{thm}{Theorem}
\newtheorem{lemma}[thm]{Lemma}
\newtheorem{corollary}[thm]{Corollary}
\newtheorem{prop}[thm]{Proposition}

\theoremstyle{definition}

\newtheorem{defini}{Definition}

\newtheorem{assumption}{Assumption}

\newenvironment{itemize*}%
{\begin{itemize}[leftmargin=*,topsep=0pt]%
		\setlength{\itemsep}{0pt}%
		\setlength{\parskip}{0pt}}%
	{\end{itemize}}
\newenvironment{enumerate*}%
{\begin{enumerate}[leftmargin=*,topsep=0pt]%
		\setlength{\itemsep}{0pt}%
		\setlength{\parskip}{0pt}}%
	{\end{enumerate}}

\newcount\Comments  
\definecolor{darkred}{rgb}{0.7,0,0}
\definecolor{teal}{rgb}{0.3,0.8,0.8}
\definecolor{forestgreen}{rgb}{0.13, 0.55, 0.13}
\newcommand{\kibitz}[2]{\ifnum\Comments=1{\textcolor{#1}{\textsf{\footnotesize #2}}}\fi}

\title{Low-Rank Generalized Linear Bandit Problems}

\author{Yangyi Lu\\
	Department of Statistics\\
	University of Michigan\\
	\texttt{yylu@umich.edu}
	\and Amirhossein Meisami \\
	Adobe Inc.\\
	\texttt{meisami@adobe.com}\\
	\and Ambuj Tewari\\
	Department of Statistics\\
	University of Michigan \\
	\texttt{tewaria@umich.edu}
}
\begin{document}

\maketitle

\begin{abstract}
In a low-rank linear bandit problem, the expected reward of an action (represented by a matrix of size $d_1 \times d_2$) is the inner product between the action and an unknown low-rank matrix $\Theta^*$.
We propose an algorithm based on a novel combination of online-to-confidence-set conversion~\citep{abbasi2012online} and the exponentially weighted average forecaster constructed by a covering of low-rank matrices. In $T$ rounds, our algorithm achieves $\widetilde{O}((d_1+d_2)^{3/2}\sqrt{rT})$ regret that improves upon the standard linear bandit regret bound of $\widetilde{O}(d_1d_2\sqrt{T})$ when the rank of $\Theta^*$: $r \ll \min\{d_1,d_2\}$. We also extend our algorithmic approach to the generalized linear setting to get an algorithm which enjoys a similar bound under regularity conditions on the link function.
To get around the computational intractability of covering based approaches, we propose an efficient algorithm by extending the "Explore-Subspace-Then-Refine" algorithm of~\citet{jun2019bilinear}. Our efficient algorithm achieves $\widetilde{O}((d_1+d_2)^{3/2}\sqrt{rT})$  regret under a mild condition on the action set $\mathcal{X}$ and the $r$-th singular value of $\Theta^*$.
Our upper bounds match the conjectured lower bound of \cite{jun2019bilinear} for a subclass of low-rank linear bandit problems. Further, we show that existing lower bounds for the sparse linear bandit problem strongly suggest that our regret bounds are unimprovable. 
To complement our theoretical contributions, we also conduct experiments to demonstrate that our algorithm can greatly outperform the performance of the standard linear bandit approach when $\Theta^*$ is low-rank.
\end{abstract}
\section{Introduction}\label{sec:intro}
Low-rank models are widely used in various applications, such as matrix completion, computer vision, etc~\citep{candes2009exact,basri2003lambertian}.
%
We study low-rank (generalized) linear models in the bandit setting~\citep{lai1985asymptotically}.
During the learning process, the agent adaptively pulls an arm (denoted as $X_t$) from a set of arms based on the past experience. 
At each pull, the agent observes a noisy reward corresponding to the arm pulled. 
Let $\Theta^*\in\mathbb{R}^{d_1\times d_2}$ be an unknown low-rank matrix with rank $r \ll \min\{d_1,d_2\}$. 
The learner's goal is to maximize the total reward:
$
\sum_{t=1}^T\mu\left(\left\langle \Theta^*, X_t \right\rangle\right)
$
where $T$ is the time horizon, $X_t \in \mathbb{R}^{d_1 \times d_2}$ is an action pulled at time $t$ that belongs to a pre-specified action set $\mathcal{X}$ and $\mu(\cdot)$ denotes a link function. Note that in the standard linear case the link function is identity. 

Many practical applications can be framed in this low-rank bandit model, where the rank of arm features has no restriction.
For traveling websites, the recommendation system needs to choose a flight-hotel bundle for the customer that can achieve high revenue. 
Often one has $m$ features of size $d_1$ for a flight ($x_1,\ldots,x_m\in\mathbb{R}^{d_1}$) and $m$ features of size $d_2$ for a hotel ($y_1,\ldots,y_m\in\mathbb{R}^{d_2}$). It is natural to form a $d_1 \times d_2$ matrix feature via outer products summation $\sum_{i=1}^m x_iy_i^T$ for each bundle, the rank of which can be any value in $\{0,1,\ldots,\min\{m,d_1,d_2\}\}$.
One can model the appeal of a bundle by a (generalized) linear function of the matrix feature $\sum_{i=1}^m x_iy_i^T$.
In online advertising with image recommendation, the advertiser selects an image to display and the goal is to achieve the maximum clicking rate.
The image is often stored as a $d_1\times d_2$ matrix, and one can use a generalized linear model (GLM) with the link function being the logistic function to model the click rate~\citep{richardson2007predicting,mcmahan2013ad}.
In all of these applications, one puts some capacity control on the underlying matrix linear coefficient $\Theta^*$ and a natural condition is $\Theta^*$ being low-rank.
We note that the examples such as online dating and online shopping discussed in \cite{jun2019bilinear} can also be formulated as our model. 
In this paper, we measure the quality of an algorithm in terms of its cumulative regret\footnote{See Section~\ref{sec:setup} for the definition.}.
A naive approach is to ignore the low-rank structure and directly apply the standard (generalized) linear bandit algorithms~\citep{abbasi2011improved,filippi2010parametric}.
These approaches suffer $\widetilde{O}(d_1d_2\sqrt{T})$ regret.\footnote{$\widetilde{O}$ omits poly-logarithmic factors of $d_1,d_2,r,T$.}
However, in practice, $d_1d_2$ can be huge.
Then a natural question is:
\begin{center}
 \emph{Can we  utilize the low-rank structure of $\Theta^*$ to achieve $o(d_1d_2\sqrt{T})$ regret?}
 \end{center}

\citet{jun2019bilinear} studied a \emph{subclass} of our problem, where the actions are {\em rank one} matrices. They proposed an algorithm that achieves $\widetilde{O}((d_1+d_2)^{3/2}\sqrt{rT})$ regret under additional incoherence and singular value assumptions of an augmented matrix defined via the arm set and $\Theta^*$ and a singular value assumption of $\Theta^*$. They also provided strong evidence that their bound is unimprovable.

We summarize our contributions below.
\begin{enumerate*}
\item We propose Low Rank Linear Bandit with Online Computation algorithm (LowLOC) for the low-rank linear bandit problem, that achieves $\widetilde{O}((d_1+d_2)^{3/2}\sqrt{rT})$ regret.
Notably, comparing with the result in \cite{jun2019bilinear}, our result
\begin{itemize}
    \item applies to more general action sets which can contain high-rank matrices and
    \item does not require the incoherence and bounded eigenvalue assumption of the augmented matrix mentioned in the previous paragraph.
\end{itemize}
Our regret bound also matches with their conjectured lower bound.
For LowLOC, we first design a novel online predictor which uses an \emph{exponentially weighted average forecaster} on a covering of low-rank matrices to solve the online low-rank linear prediction problem with $O((d_1+d_2) r \log T)$ regret.
We then plug in our online predictor to the online-to-confidence-set conversion framework proposed by \citet{abbasi2012online} to construct a confidence set of $\Theta^*$ in our bandit setting, and at every round we choose the action optimistically.

\item We further propose Low Rank Generalized Linear Bandit with Online Computation algorithm (LowGLOC) for the generalized linear setting that also achieves $\widetilde{O}((d_1+d_2)^{3/2}\sqrt{rT})$ regret.
LowGLOC is similar to LowLOC but here we need to design a new online-to-confidence-set conversion method, which can be of independent interest.



\item LowLOC and LowGLOC enjoy good regret but are unfortunately not efficiently implementable.
To overcome this issue, we provide an efficient algorithm Low-Rank-Explore-Subspace-Then-Refine (LowESTR) for the linear setting, inspired by the ESTR algorithm proposed by \cite{jun2019bilinear}.
We show that under a mild assumption on action set $\mathcal{X}$, LowESTR achieves $\widetilde{O}((d_1+d_2)^{3/2}\sqrt{rT}/\omega_r)$ regret, where $\omega_r>0$ is a lower bound for the $r$-th singular value of $\Theta^*$.
Comparing with ESTR, LowESTR does not need the incoherence and the eigenvalue assumption of the augmented matrix while the assumptions on the action set of the two algorithms are different.
We also provide empirical evaluations to demonstrate the effectiveness of LowESTR.

\end{enumerate*}

\section{Related Work}\label{sec:rel}
Our work is inspired by \cite{jun2019bilinear} where they model the reward as $x_t^\top \Theta^* z_t$. $x_t \in \mathcal{X} \subset \mathbb{R}^{d_1}$ is a left arm and $z_t \in \mathcal{Z} \subset \mathbb{R}^{d_2}$ is a right arm ($\mathcal{X}$ and $\mathcal{Z}$ are left and right arm sets, repsectively).
Note this model is a special case of our low-rank linear bandit model because one can write $x_t^\top \Theta^* z_t = \left\langle\Theta^*, x_t z_t^\top\right\rangle$ and define the arm set as $\mathcal{X}\mathcal{Z}^\top$.
Their ESTR algorithm enjoys $O((d_1+d_2)^{3/2}\sqrt{rT}/\omega_r)$ regret bound under the assumptions: 1) an augmented matrix $K^* = X\Theta^*Z^\top$ is incoherent~\citep{keshavan2010matrix} and has a finite condition number, where $X \in \mathbb{R}^{d_1 \times d_1}$ is constructed by $d_1$ arms from $\mathcal{X}$ that maximizes $\|X^{-1}\|_{2}$ and $Z \in \mathbb{R}^{d_2 \times d_2}$ is constructed by $d_2$ arms from $\mathcal{Z}$ that maximizes $\|Z^{-1}\|_{2}$, and 2) $\|X^{-1}\|_{2}$ and  $\|Z^{-1}\|_{2}$ are upper bounded by a constant.
Their algorithm requires explicitly finding $X$ and $Z$, which is in general NP-hard, even though they also proposed heuristics to speed up this step.
Comparing with ESTR, our LowLOC and LowGLOC algorithm are also not computationally efficient, but they both apply to richer action sets (matrices of any rank) without assumptions on $K^*$, $X$ and $Z$ and their regret bound does not depend on $\omega_r$.
Our LowESTR algorithm is computationally efficient if the action set admits a nice exploration distribution (see details in Section~\ref{sec:ESTR}). LowESTR achieves $O\left((d_1+d_2)^{3/2}\sqrt{rT}/\omega_r\right)$ regret bound but it does not require assumptions on $K^*$, $X$ and $Z$ as well.


\citet{katariya2017stochastic} and \citet{kveton2017stochastic} also studied rank-1 and low-rank bandit problems. 
They assume there is an underlying expected reward matrix $\bar{R}$, at each time the learner picks an element on $(i_t,j_t)$ position and receives a noisy reward. 
It can be viewed as a special case of bilinear bandit with one-hot vectors as left and right arms. 
\citet{katariya2017stochastic} is further extended by \citet{katariya2017bernoulli} that uses KL based confidence intervals to achieve a tighter regret bound.
Our problem is more general comparing to these works.
\citet{johnson2016structured} considered the same setting as ours, but their method relies on the knowledge of many parameters that depend on the unknown $\Theta^*$ and in particular only works for continuous arm set. 

There are other works that utilize the low-rank structure in different model settings. 
For example, \citet{gopalan2016low} studied low rank bandits with latent structures using robust tensor power method. 
\citet{lale2019stochastic} imposed low-rank assumptions on the feature vectors to reduce the effective dimension. 
These work all utilize the low-rank structure to achieve better regret bound than standard approaches that do not take the low-rank structure into account.
\section{Preliminaries}\label{sec:setup}
We formally define the problem and review relevant background in this section.
\subsection{Low-rank Linear Bandit}
Let $\mathcal{X}\subset\mathbb{R}^{d_1\times d_2}$ be the arm space. In each round $t$, the learner chooses an arm $X_t\in\mathcal{X}$, and observes a noisy reward of a linear form:
\begin{align*}
	y_t = \langle X_t, \Theta^*\rangle+\eta_t,
\end{align*}
where $\Theta^*\in\mathbb{R}^{d_1\times d_2}$ is an unknown parameter and $\eta_t$ is a $1$-sub-Gaussian random variable. Denote the rank of $\Theta^*$ by $r$, we assume $r \ll \min\{d_1,d_2\}$.
Let the $r$-th singular value of $\Theta^*$  is lower bounded by $\omega_r>0$. 
We use $\left\langle A, B \right\rangle:=\mathbf{trace}(A^TB)$ to denote the inner product between matrix $A$ and $B$.
We follow the standard assumptions in linear bandits:
\begin{align*}
    \left\|\Theta^*\right\|_F\leq 1 \text{ and } \left\|X\right\|_F\leq 1, \text{ for all } X\in\mathcal{X}.
\end{align*}

In this low-rank linear bandit problem, the goal of the learner is to maximize the total reward $\sum_{t=1}^{T}\langle X_t, \Theta^*\rangle$, where $T$ is the time horizon. Clearly, with the knowledge of the unknown parameter $\Theta^*$, one should always select an action $X^* \in \argmax_{X\in\mathcal{X}} \langle X, \Theta^*\rangle$. It is natural to evaluate the learner relative to the optimal strategy.
The difference between the learner's total reward and the total reward of the optimal strategy is called \emph{pseudo-regret}~\citep{audibert2009exploration}:
\begin{align*}
    R_T := \sum_{t=1}^{T}\langle X^*-X_t,\Theta^*\rangle.
\end{align*}
For simplicity, we use the word regret instead of pseudo-regret for $R_T$.

\subsection{Generalized Low-rank Linear Bandit}
We also study the \emph{generalized linear bandit model} of the following form:
$\mathbb{E}\left[y_t|X_t,\Theta^*\right] = \mu\left(\langle X_t,\Theta^*\rangle\right)$ where $\mu\left(\cdot\right)$ is a link function.
This framework builds on the well-known Generalized Linear Models (GLMs) and has been widely studied in many applications. For example, when rewards are binary-valued, a natural link function is the logistic function $\mu(x) = \exp(x)/(1+\exp(x))$.
For the generalized setting, we assume the reward given the action follows an exponential family distribution:
\begin{align}
\mathbb{P}\left(y|z = \langle X,\Theta^*\rangle\right) = \exp\left(\frac{yz-m(z)}{\phi(\tau)}+h(y,\tau)\right), \label{equ:GLM}
\end{align}
where $\tau\in\mathbb{R}^+$ is a known scale parameter and $m,\phi$ and $h$ are some known functions. From basic calculation we get $m'(z) = \mathbb{E}[y|z]:=\mu(z)$.
We assume the above exponential family is a minimal representation, then $m(z)$ is ensured to be strictly convex~\citep{wainwright2008graphical}, and thus the negative log likelihood (NLL) loss $\ell(z,y):=-yz+m(z)$ is also strictly convex.

We make the following standard assumption on the link function $\mu(\cdot)$~\citep{jun2017scalable}.
\begin{assumption}\label{assump:GLOC}
    There exist constants $L_\mu, c_\mu\geq 0,\kappa_\mu>0$, such that
	the link function $\mu(\cdot)$ is $L_\mu-$Lipschitz on $[-1,1]$, continously differentiable on $(-1,1)$, $\inf_{z\in(-1,1)}\mu'(z) := \kappa_\mu$ and $|\mu(0)|\leq c_\mu$.
\end{assumption}
One can write down the above reward model \eqref{equ:GLM} in an equivalent way:
\begin{align*}
    y_t = \mu\left(\langle X_t,\Theta^*\rangle\right)+\eta_t,
\end{align*}
where $\eta_t$ is conditionally $R$-sub-Gaussian given $X_t$ and $\{(X_s,\eta_s)\}_{s=1}^{t-1}$. Using the form of $\mathbb{P}(y|z)$, Taylor expansion and the strictly convexity of $m(\cdot)$, one can show that $R = \sup_{z\in[-1,1]}\sqrt{\mu'(z)}\leq \sqrt{L_\mu}$ by the definition of the sub-Gaussian constant.
An optimal arm is $X^* \in \argmax_{X\in\mathcal{X}}\mu\left(\left\langle X,\Theta^*\right\rangle\right)$. The performance of an algorithm is again evaluated by cumulative regret:
\begin{align*}
    R_T = \sum_{t=1}^{T}\mu\left(\left\langle X^*, \Theta^*\right\rangle\right)-\mu\left(\langle X_t, \Theta^*\rangle\right).
\end{align*}

\paragraph{Other notations.} We use $O$ and $\Omega$ for the standard Big O and Big Omega notations. $\widetilde{O}$ and $\widetilde{\Omega}$ ignore the poly-logarithmic factors of $d_1,d_2,r,T$. $f(x)\asymp g(x)$ indicates $f$ and $g$ are of the same order ignoring the poly-logarithmic factors of $d_1,d_2,r,T$.
For any set $\mathcal{S}$, we use $|\mathcal{S}|$ to denote its cardinality.

\section{Low-rank Linear Bandit with Online Computation}\label{sec:LowLOC}
\begin{algorithm}[t]
	\caption{Low-Rank Linear Bandit with Online Computation (LowLOC)}
	\label{algo:LowLOC}
	\begin{algorithmic}[1]
		\STATE \textbf{Input: }arm set: $\mathcal{X}$, horizon: $T$, $\frac{1}{T}$-net for $S_r$: $\bar{S}_r(\frac{1}{T})$, failure rate $\delta$, EW constant $\eta \asymp \frac{1}{\log\left(T/\delta\right)}$.
		\STATE Initial confidence set \\
		$C_0 = \{\Theta\in\mathbb{R}^{d_1\times d_2}: \left\|\Theta\right\|_F^2\leq 1\}$.
		\FOR{$t=1,\ldots,T$}
		\STATE $(X_t,\widetilde{\Theta}_t):=\argmax_{(X,\Theta)\in \mathcal{X}\times C_{t-1}} \langle X,\Theta\rangle$.
		\STATE Pull arm $X_t$ and receive reward $y_t$.
		\STATE Compute EW predictor\\
		$\hat{y}_t = \frac{\sum_{i=1}^{|\bar{S}_r(\frac{1}{T})|}e^{-\eta L_{i,t-1}}f_{\Theta_i,t}}{\sum_{j=1}^{|\bar{S}_r(\frac{1}{T})|}e^{-\eta L_{j,t-1}}}$, \\
		where $f_{\Theta_i,t} \triangleq \left\langle X_t,\Theta_i\right\rangle$ for $\Theta_i\in\bar{S}_r(\frac{1}{T})$.
		\STATE Update losses $L_{i,t} = \sum_{s=1}^t(y_s-f_{\Theta_i,s})^2$, for $i=1,\ldots,|\bar{S}_r(\frac{1}{T})|$.
		\STATE Update $\mathcal{C}_t$ according to Equation~\eqref{equ:CI}, where $B_t$ is defined in Lemma~\ref{lemma:EW}.
		\ENDFOR
	\end{algorithmic}
\end{algorithm}
We first present our algorithm, LowLOC (Algorithm~\ref{algo:LowLOC}) for low-rank linear bandit problems. 
Before diving into details, we summarize our results as follows:
\begin{thm}[Regret of LowLOC (Algorithm~\ref{algo:LowLOC})] \label{thm:LowLOC-EWAF}
	For $\forall \delta\in(0,0.25]$, with probability at least $1-\delta$, Algorithm~\ref{algo:LowLOC} achieves regret:
	\begin{align*}
	R_T = \widetilde{O}\left((d_1+d_2)^{3/2}\sqrt{rT}\sqrt{\log\left(\frac{1}{\delta}\right)}\right).
	\end{align*}
\end{thm}
Note that LowLOC achieves the desired goal of  outperforming the standard linear bandit approach with $\widetilde{O}(d_1d_2\sqrt{T})$ regret.
Furthermore, this bound does not depend on any other problem-dependent parameters such as least singular value of $\Theta^*$ and does not require any other assumption which appeared in \cite{jun2019bilinear}.
In the following sub-sections, we explain details of our algorithm design choices.

\subsection{OFU and Online-to-confidence-set Conversion}
This algorithm follows the standard Optimism in the Face of Uncertainty (OFU) principle.
We maintain a confidence set $\mathcal{C}_t$ at every round that contains the true parameter $\Theta^*$ with high probability and we choose the action $X_t$ according to
\begin{align*}
    (X_t,\widetilde{\Theta}_t) = \argmax_{(X,\Theta)\in \mathcal{X}\times \mathcal{C}_{t-1}}\langle X,\Theta\rangle.
\end{align*}

Typically, the faster $\mathcal{C}_t$ shrinks, the lower regret we have.
The main diffculty is to construct $\mathcal{C}_t$ that leverages the low-rank structure so that we only have $\widetilde{O}((d_1+d_2)^{3/2}\sqrt{rT})$ regret.
Our starting point is to use the online-to-confidence-set conversion framework proposed by \cite{abbasi2012online} who builds the confidence set based on an online predictor.
At each round, an online predictor receives $X_t$, predicts $\hat{y}_t$, based on historical data $\{(X_s,y_s)\}_{s=1}^{t-1}$, observes the true value $y_t$ and suffers a loss $\ell_t\left(\hat{y}_t\right) \triangleq \left(y_t-\hat{y}_t\right)^2$.
The performance of this online predictor is measured by comparing its cumulative loss to the cumulative loss of a fixed linear predictor using coefficient $\Theta$: 
\begin{align*}
    \rho_t(\Theta) = \sum_{s=1}^{t}\ell_s(\hat{y}_s)-\ell_s(\langle\Theta,X_s\rangle).
\end{align*}

The key idea of online-to-confidence-set conversion (adapted to our low-rank setting) is that if one can guarantee $\sup_{\left\|\Theta\right\|_F\leq 1,\text{rank}(\Theta)\leq r}\rho_t(\Theta)\leq B_t$ for some non-decreasing sequence $\{B_t\}_{t=1}^T$, we can use this information to construct the confidence interval for $\Theta^*$ as:
\begin{align}
	\mathcal{C}_{t} &= \{\Theta\in\mathbb{R}^{d_1\times d_2}:\notag\\
	&\left\|\Theta\right\|_F^2+\sum_{s=1}^{t}(\hat{y}_s-\langle\Theta,X_s\rangle)^2\leq 1+\beta_t(\delta)\},
	\label{equ:CI}
\end{align}
where $\beta_t(\delta) = 1+2B_t+32\log\left(\left(\sqrt{8}+\sqrt{1+B_t}\right)/\delta\right)$ and $\delta$ is the failure probability. 

Lemma~\ref{lemma:LOC} in appendix guarantees that $\Theta^*$ is contained in $\cap_{t\geq 1} \mathcal{C}_t$ with high probability and Lemma~\ref{lemma:regret_LowLOC} further guarantees the overall regret 
\begin{align}
\label{equ:EW_regret_incomplete}
    R_T = \widetilde{O}(\sqrt{d_1d_2\beta_{T-1}(\delta) T}) = \widetilde{O}\left(\left(d_1+d_2\right)\sqrt{B_{T-1}T}\right).
\end{align}

Therefore, the problem to achieve the $\widetilde{O}((d_1+d_2)^{3/2}\sqrt{rT})$ regret bound reduces to designing an online predictor which guarantees $\sup_{\left\|\Theta\right\|_F\leq 1,\text{rank}(\Theta)\leq r}\rho_t(\Theta)\leq B_t$ and $B_t = \widetilde{O}\left((d_1+d_2)r\right)$.
To achieve this rate, the key is to leverage the low-rank structure of $\Theta^*$.

\subsection{Online Low Rank Linear Prediction}
We adopt the classical \emph{exponentially weighted average forecaster}  (EW) framework~\citep{cesa2006prediction} which uses $N$ experts  to predict $\hat{y}_t$ with the following formula
\begin{align}
\widehat{y}_t = \frac{\sum_{i=1}^{N}e^{-\eta L_{i,t-1}}f_{i,t}}{\sum_{j=1}^{N}e^{-\eta L_{j,t-1}}}. \label{equ:exp_forecaster}
\end{align}
In above, $f_i$ denotes the $i$-th expert that makes a prediction $f_{i,t}$ at time $t$,
$L_{i,t-1}\triangleq\sum_{s=1}^{t-1}\ell_s(f_{i}\left(X_t\right))$ is the cumulative loss incurred by expert $i$, and $\eta$ is a tuning parameter.
By choosing $\eta$ carefully, one can guarantee that this predictor achieves $O\left(\log N \log(T/\delta)\right)$ regret comparing with the best expert among the expert set.
See backgrounds on the construction of EW in Section~\ref{sec:app_EW_prelim.tex} and Proposition 3.1 in~\citet{cesa2006prediction}.



In our setting, an expert can be viewed as a matrix $\Theta$ that satisfies $\|\Theta\|_F \le 1$ and $\text{rank}\left(\Theta\right) \le r$, and makes prediction according to $f_{\Theta,t}\triangleq\langle \Theta, X_t\rangle$.
There are infinitely many such experts and therefore we cannot directly use EW which requires finite number of experts.
Our main idea is to construct $N$ experts which guarantees $\log N$ is small and these $N$ experts can represent the original expert set $S_r\triangleq \{\Theta \in \mathbb{R}^{d_1 \times d_2}:\|\Theta\|_F \le 1, \text{rank}\left(\Theta\right) \le r\}$ well, and then apply EW using these $N$ experts.
We construct an $\varepsilon$-net $\bar{S}_r(\varepsilon)$, i.e., for any $\Theta\in S_r$, there exists a $\bar{\Theta}\in\bar{S}_r(\varepsilon)$, such that $\left\|\Theta-\bar{\Theta}\right\|_F\leq \epsilon$. 
We further prove that $|\bar{S}_r(\varepsilon)|\leq (9/\varepsilon)^{(d_1+d_2+1)r}$ in Lemma~\ref{lemma:cover}, so the number of experts $N$ in Equation~\eqref{equ:exp_forecaster} is at most $(9T)^{(d_1+d_2+1)r}$ if we set $\varepsilon = 1/T$.

The following lemma summarizes the performance of this online predictor.
\begin{lemma}[Regret of EW under Squared Loss] \label{lemma:EW}
	Let $\eta = \frac{1}{2(2+\sqrt{2\log\left(2T/\delta\right)})^2}$ in EW forecaster~\eqref{equ:exp_forecaster}. Then, for any $0<\delta<0.25$, with probability at least $1-\delta$, we have
\begin{align*}
 \sup_{\left\|\Theta\right\|_F\leq 1,\text{rank}(\Theta)\leq r}\rho_T(\Theta) = \widetilde{O}\left((d_1+d_2)r\log\left(\frac{1}{\delta}\right)\right).
\end{align*}
\end{lemma}
To obtain Theorem~\ref{thm:LowLOC-EWAF}, one just needs to plug Lemma~\ref{lemma:EW} into Equation~\eqref{equ:EW_regret_incomplete} by defining $B_T$ as $\sup_{\left\|\Theta\right\|_F\leq 1,\text{rank}(\Theta)\leq r}\rho_T(\Theta)$.
\section{Low-rank Generalized Linear Bandit}\label{sec:LowGLOC}
\begin{algorithm}[t]
	\caption{Low-rank Generalized Linear Bandit with Online Computation (LowGLOC)}
	\label{algo:LowGLOC}
	\begin{algorithmic}[1]
	    
	    \STATE \textbf{Input: }arm set: $\mathcal{X}$, horizon: $T$, $\frac{1}{T}$-net for $S_r$: $\bar{S}_r(\frac{1}{T})$, failure rate $\delta$, EW constant $\eta \asymp \frac{1}{\log\left(T/\delta\right)}$, function $m(\cdot)$ in the generalized linear model.
		\STATE Initial confidence set\\
		$C_0 = \{\Theta\in\mathbb{R}^{d_1\times d_2}: \left\|\Theta\right\|_F^2\leq 1\}$.
		\FOR{$t=1,\ldots,T$}
		\STATE $(X_t,\widetilde{\Theta}_t):=\argmax_{(X,\Theta)\in \mathcal{X}\times C_{t-1}} \langle X,\Theta\rangle$.
		\STATE Pull arm $X_t$ and receive reward $y_t$.
		\STATE Compute EW predictor\\
		$\hat{y}_t = \frac{\sum_{i=1}^{|\bar{S}_r(\frac{1}{T})|}e^{-\eta L_{i,t-1}}f_{\Theta_i,t}}{\sum_{j=1}^{|\bar{S}_r(\frac{1}{T})|}e^{-\eta L_{j,t-1}}}$,\\ where $f_{\Theta_i,t} \triangleq \left\langle X_t,\Theta_i\right\rangle$ for $\Theta_i\in\bar{S}_r(\frac{1}{T})$.
		\STATE Update losses 
		$L_{i,t} = \sum_{s=1}^t-f_{\Theta_i,s}y_s+m(f_{\Theta_i,s})$, for $i=1,\ldots,|\bar{S}_r(\frac{1}{T})|$.
		\STATE Update $\mathcal{C}_t$ according to Equation~\eqref{equ:CI_GLB}, where $B_t^{\text{GLB}}$ is as defined in Lemma~\ref{lemma:EW_NLL}.
		\ENDFOR
	\end{algorithmic}
\end{algorithm}
We also study the low-rank generalized linear bandit setting. The main structure of our algorithm LowGLOC (Algorithm~\ref{algo:LowGLOC}) is similar to LowLOC, so we focus on the key differences in this section.


We still use EW to perform online predictions, but instead of the squared loss, we use negative log likelihood (NLL) loss $\ell_s(\hat{y}_s) = -\hat{y}_sy_s+m(\hat{y}_s)$ to construct the forecaster in Equation~\eqref{equ:exp_forecaster}, where $m(\cdot)$ is as defined in Section~\ref{sec:setup}.
Therefore, the performance of EW using NLL loss relative to a fixed linear predictor $\Theta$ is measured by:
\begin{align*}
	\rho_T^{\text{GLB}}(\Theta) &= \left(\sum_{t=1}^{T}-\hat{y}_ty_t+m(\hat{y}_t)\right)\\
	&-\left(\sum_{t=1}^{T}-\langle\Theta,X_t\rangle y_t+m(\langle\Theta,X_t\rangle)\right).
\end{align*}
If there exists a non-decreasing sequence $\{B_t^{\text{GLB}}\}_{t=1}^T$ such that $\sup_{\left\|\Theta\right\|_F\leq 1,\text{rank}(\Theta)\leq r}\rho_t^{\text{GLB}}(\Theta) \leq B_t^{\text{GLB}}$, we construct $\mathcal{C}_t^{\text{GLB}}$ in the following way:
\begin{align}
	\mathcal{C}_t^{\text{GLB}} &= \{\Theta\in\mathbb{R}^{d_1\times d_2}: \notag\\ &\left\|\Theta\right\|_F^2+\sum_{s=1}^{t}\left(\hat{y}_s-\langle\Theta^*,X_s\rangle\right)^2\leq \beta_t^{\text{GLB}}(\delta)\}, \label{equ:CI_GLB}
\end{align}
where
\begin{align*}
    \beta_t^{\text{GLB}}(\delta) &= 2+\frac{4}{\kappa_\mu}B_t^{\text{GLB}}\\
    &+\frac{32L_\mu}{\kappa_\mu^2}\log\left(\left(\sqrt{L_\mu}\sqrt{\frac{8}{\kappa_\mu^2}}+\sqrt{\frac{2}{\kappa_\mu}B_t^{\text{GLB}}+1}\right)\frac{1}{\delta}\right),
\end{align*}
and $\delta$ is the failure probability. 

Lemma~\ref{lemma:GLOC} guarantees that the true parameter $\Theta^*$ is contained in $\cap_{t\geq 1} \mathcal{C}_t^{\text{GLB}}$ with high probability.
Lemma~\ref{lemma:regret_LowGLOC} further guarantees that the overall regret of LowGLOC satisfies
\begin{align*}
    R_T = \widetilde{O}(\sqrt{d_1d_2\beta_{T-1}^{\text{GLB}}(\delta)T}) = \widetilde{O}((d_1+d_2)\sqrt{B_T^{\text{GLB}}T/\kappa_\mu}).
\end{align*}
Following the online-to-confidence-set conversion idea as used in LowLOC, we prove that
\begin{align*}
    B_T^{\text{GLB}} = O\left(\frac{L_\mu^2+c_\mu^2}{\kappa_\mu}(d_1+d_2)r\log T\log\left(\frac{T}{\delta}\right)\right)
\end{align*}
in Lemma~\ref{lemma:EW_NLL}.

We next present the regret of LowGLOC in the following theorem, which can be easily achieved by plugging Lemma~\ref{lemma:EW_NLL} into Lemma~\ref{lemma:regret_LowGLOC} as described in above paragraph.
\begin{thm}[Regret of LowGLOC] \label{thm:LowGLOC-EWAF}
	For $\forall \delta\in(0,0.25]$, with probability at least $1-\delta$, Algorithm~\ref{algo:LowGLOC} achieves regret:
	\begin{align*}
		R_T = \widetilde{O}\left((d_1+d_2)^{3/2}\sqrt{\frac{L_\mu^2+c_\mu^2}{\kappa_\mu^2}rT\log\left(\frac{1}{\delta}\right)}\right).
	\end{align*}
\end{thm}
To the best of our knowledge, this is the first algorithm that achieves $o(d_1d_2\sqrt{T})$ regret bound for low-rank GLM bandits. 

\section{An Efficient Algorithm for the Linear Case}\label{sec:ESTR}
At every round, LowLOC and LowGLOC need to calculate exponentially weighted predictions, which involves calculating weights of the covering of low-rank matrices. These approaches have high computation complexity even though their regret is ideal. In this section, we propose a computationally efficient method LowESTR (Algorithm~\ref{algo:ESTR}) that also achieves
$\widetilde{O}((d_1+d_2)^{3/2}\sqrt{rT})$ regret under mild assumptions on the action set $\mathcal{X}$ as follows.
\begin{assumption}\label{assump:ESTR}
	There exists a sampling distribution $D$ over $\mathcal{X}$ with covariance matrix $\Sigma$, such that $\lambda_{min}(\Sigma)\asymp\frac{1}{d_1d_2}$ and $D$ is sub-Gaussian with parameter $\sigma^2\asymp\frac{1}{d_1d_2}$. (see Definition~\ref{def:subG} in Section~\ref{sec:app_ESTR} for the definition of sub-Gaussian random matrices.)
\end{assumption}
This assumption is easily satisfied in many arm sets.
To guarantee the existence of above sampling distribution $D$, we only need the convex hull of a subset of arms $\mathcal{X}_{sub}\subset\mathcal{X}$ contains a ball with radius $R\leq 1$, which does not scale with $d_1$ or $d_2$.
For example, if  $\mathcal{X}$ is the Euclidean unit ball/sphere in $\mathbb{R}^{d_1\times d_2}$, we can simply set $D$ to be the uniform distribution over $\mathcal{X}$. 
Notably, different choices of $D$ satisfying Assumption~\ref{assump:ESTR} do not affect the overall regret.

We extend the two-stage procedure "Explore Subspace Then Refine (ESTR)" proposed by~\cite{jun2019bilinear}. In stage 1, ESTR estimates the row and column subspaces of $\Theta^*$. In stage 2, ESTR transforms the original problem into a $d_1d_2$-dimensional linear bandit problem and invokes LowOFUL algorithm (Algorithm~\ref{algo:LowOFUL})~\citep{jun2019bilinear}, which leverages the estimated row/column subspaces of $\Theta^*$. 

\begin{algorithm*}[tbp]
	\caption{Low Rank Explore Subspace Then Refine (LowESTR)}
	\label{algo:ESTR}
	\begin{algorithmic}[1]
		\STATE \textbf{Input:} arm set $\mathcal{X}$, time horizon $T$, exploration length $T_1$, rank $r$ of $\Theta^*$, spectral bound $\omega_r$ of $\Theta^*$, sampling distribution for stage 1: $D$; parameters for LowOFUL in stage 2: $B,B_\perp,\lambda,\lambda_\perp$.
		\STATE \textbf{Stage 1: Explore the Low Rank Subspace}
		\STATE Pull $X_t\in\mathcal{X}$ according to distribution $D$ and observe reward $Y_t$, for $t = 1,\ldots,T_1$.
		\STATE Solve $\widehat{\Theta}$ using the problem below:
		\begin{align}
		\widehat{\Theta} = \argmin_{\Theta\in\mathbb{R}^{d_1\times d_2}} \frac{1}{2T_1}\sum_{t=1}^{T_1}\left(Y_t - \langle X_t, \Theta\rangle\right)^2+\lambda_{T_1}\left\|\Theta\right\|_{\text{nuc}} \label{equ:nuc}.
		\end{align}
		\STATE Let $\widehat{\Theta} = U\widehat{S}V^T$ be the SVD of $\widehat{\Theta}$. Take the first $r$ columns of $U$ as $\widehat{U}$, the first $r$ rows of $V$ as $\widehat{V}$. Let $\widehat{U}_\perp$ and $\widehat{V}_\perp$ be orthonormal bases of the complementary subspaces of $\widehat{U}$ and $\widehat{V}$.
		\STATE \textbf{Stage 2: Refine Standard Linear Bandit Algorithm}
		\STATE Rotate the arm feature set: $\mathcal{X'}:=\{[\widehat{U} \ \widehat{U}_\perp]^TX[\widehat{V} \ \widehat{V}_\perp]: X\in\mathcal{X}\}$.
		\STATE Define a vectorized arm feature set so that the last $(d_1-r)(d_2-r)$ components are from the complementary subspaces:
		\begin{align*}
		\mathcal{X'}_{\text{vec}}:=\{\text{vec}(X_{1:r,1:r}');\text{vec}(X_{r+1:d_1,1:r}');\text{vec}(X_{1:r,r+1:d_2}');\text{vec}(X_{r+1:d_1,r+1:d_2}'):X'\in\mathcal{X'}\}.
		\end{align*}
		\STATE For $T_2 = T-T_1$ rounds, invoke LowOFUL (Algorithm~\ref{algo:LowOFUL}) with arm set $\mathcal{X'}_{\text{vec}}$, the low dimension $k=(d_1+d_2)r-r^2$ and $\gamma(T_1) \asymp \frac{(d_1+d_2)^3r}{T_1\omega_r^2}$, $B,B_\perp,\lambda,\lambda_\perp$.
	\end{algorithmic}
\end{algorithm*}

\begin{algorithm}[h]
	\caption{LowOFUL~\citep{jun2019bilinear}}
	\label{algo:LowOFUL}
	\begin{algorithmic}[1]
		\STATE \textbf{Input: }$T,k$, arm set $\mathcal{A}\subset\mathbb{R}^{d_1\times d_2}$, failure rate $\delta$ and positive constants $B,B_\perp,\lambda,\lambda_\perp$.
		\STATE $\Lambda = \textbf{diag}(\lambda,\ldots,\lambda,\lambda_\perp,\ldots,\lambda_\perp)$, where $\lambda$ occupies the first $k$ diagonal entries. 
		\FOR{$t=1,\ldots,T$}
		\STATE Compute $a_t = \argmax_{a\in\mathcal{A}}\max_{\theta\in \mathcal{C}_{t-1}}\langle\theta,a\rangle$.
		\STATE Pull arm $a_t$ and receive reward $y_t$.
		\STATE Update $C_t = \{\theta:\left\|\theta-\hat{\theta}\right\|_{V_t}\leq \sqrt{\beta_t}\}$, \\
		where $\sqrt{\beta_t} = \sqrt{\log\frac{|V_t|}{|\Lambda|\delta^2}}+\sqrt{\lambda}B+\sqrt{\lambda_\perp}B_\perp$, \\
		$V_t = \Lambda+\sum_{s=1}^{t}a_ta_t^T$, 
		\\$\hat{\theta}_t = (\Lambda+A^TA)^{-1}A^T\mathbf{y}$. \\
		(Here $A = [a_1^T;\ldots;a_t^T]$ and $\mathbf{y}:=[y_1,\ldots,y_t]^T$).
		\ENDFOR
	\end{algorithmic}
\end{algorithm}

\subsection{Description for LowESTR}
LowESTR also proceeds with the two-stage framework as ESTR, but we use different estimation method in stage 1. 
\paragraph{Stage 1.}
We are inspired by a line of work on low-rank matrices recovery using nuclear-norm penalty with squared loss \citep{wainwright2019high}.
The learner pulls arm $X_t\in\mathcal{X}$ according to distribution $D$ and observes the reward $y_t$ up to a horizon $T_1$, then uses $\{X_t,y_t\}_{t=1}^{T_1}$ to solve a nuclear-norm penalized least square problem in \eqref{equ:nuc} and receives an estimated $\widehat{\Theta}$ for $\Theta^*$. 
Notably, instead of invoking an NP-hard problem in stage 1 as ESTR, the optimization problem \eqref{equ:nuc} in LowESTR is convex and thus can be solved easily using standard gradient based methods. 
Assumption~\ref{assump:ESTR} guarantees $\left\|\widehat{\Theta}-\Theta^*\right\|_F^2 \asymp \frac{(d_1+d_2)^3r}{T_1}$ in Theorem~\ref{thm:theta_conv} (Section~\ref{sec:app_theta_conv}).
We get the estimated row/column subspaces of $\Theta^*$ simply by running an SVD step. 

\paragraph{Stage 2.}
In stage 2, we apply LowOFUL algorithm (Algorithm~\ref{algo:LowOFUL}) proposed by~\citet{jun2019bilinear} in our setting. The key idea is reducing the problem to linear bandit and utilizing the estimated subspaces in the standard linear bandit method OFUL~\citep{abbasi2011improved}.

We now present the overall regret of Algorithm~\ref{algo:ESTR}.
\begin{thm}[Regret of LowESTR for Low Rank Bandit] \label{thm:ESTR_regret}
	Suppose we run LowESTR  in stage 1 with $T_1 \asymp (d_1+d_2)^{3/2}\sqrt{rT}\frac{1}{\omega_r}$ and $\lambda_{T_1}^2 \asymp \frac{1}{T_1\min\{d_1,d_2\}}$. We invoke LowOFUL (Algorithm~\ref{algo:LowOFUL}) in stage 2 with $k = r(d_1+d_2-r)$, $\lambda_\perp = \frac{T_2}{k\log\left(1+T_2/\lambda\right)}$, $B = 1$, $B_\perp = \gamma(T_1)$, and the rotated arm sets $\mathcal{X}_{\text{vec}}'$ defined in Algorithm~\ref{algo:ESTR}, the overall regret of LowESTR is, with prob at least $1-2\delta$,
	\begin{align*}
	R_T = \widetilde{O}\left((d_1+d_2)^{3/2}\sqrt{rT}\frac{1}{\omega_r}\right).
	\end{align*}
\end{thm}
We believe that this ``Explore-Subspace-Then-Refine" framework can also be extended to the generalized linear setting. In stage 1, an M-estimator that minimizes the negative log-likelihood plus nuclear norm penalty~\citep{fan2019generalized} can be used instead, while in stage 2, one can revise a standard generalized linear bandit algorithm such as GLM-UCB~\citep{filippi2010parametric} by leveraging the low-rank knowledge in the same way as LowOFUL. We leave this extension for future work.

\subsection{Computational Complexity}
Before we end this section, we note that the computational complexity of LowESTR is polynomial in the relevant quantities.
\begin{prop}[Computational complexity of LowESTR]\label{prop:complexity}
The computational complexity of LowESTR (Algorithm~\ref{algo:ESTR}) is at most
\begin{align*}
    O\left(d_1d_2(d_1+d_2)^3rT/\omega_r^2+d_1^2d_2^2T^2+d_1^3d_2^3T\right).
\end{align*}
\end{prop}

In stage 1, we solve a convex optimization problem with unknown $\Theta\in\mathbb{R}^{d_1\times d_2}$ using subgradient method, of which the complexity is $O(T_1d_1d_2/\epsilon^2)$ ($\epsilon$ refers to the target accuracy). 
The complexity of the SVD step at the end of stage 1 is $O(d_1d_2\min\{d_1,d_2\})$. 

In stage 2, LowOFUL algorithm (Algorithm~\ref{algo:LowOFUL}) dominates the computational complexity. 
In iteration $t$ of LowOFUL, usually $a_t = \argmax_{a\in\mathcal{A}}\max_{\theta\in\mathcal{C}_{t-1}}\langle\theta,a\rangle$ can be solved with an oracle in constant time, the complexity of least square estimation is $O(d_1^2d_2^2t+d_1^3d_2^3)$ due to matrix multiplication and Cholesky factorization. Thus, in $T_2\leq T$ iterations, the computational complexity of stage 2 is at most $O(d_1^2d_2^2T^2+d_1^3d_2^3T)$. 

Combining the complexity results in two stages, taking the target accuracy $\epsilon = 1/\sqrt{T_1}$ and $T_1 = O\left((d_1+d_2)^{3/2}\sqrt{rT}\frac{1}{\omega_r}\right)$ as stated in Theorem 4, the overall computational complexity in Proposition~\ref{prop:complexity} is achieved.

\section{Lower Bound for Low-rank Linear Bandit}\label{sec:lower}
In this section, we discuss the regret lower bound of the low-rank linear bandit model.
Suppose $d_1=d_2=d$, we first present a $\Omega(dr\sqrt{T})$ lower bound, which is a straightforward extension of the linear bandit lower bound~\citep{lattimore2018bandit}.



\begin{thm}[Lower Bound]\label{thm:lower}
	Assume $dr\leq 2T$ and let $\mathcal{X} = \{X\in\mathbb{R}^{d\times d}: \left\|X\right\|_F\leq 1\}$. Then $\exists\Theta\in\mathbb{R}^{d\times d}$, where $\left\|\Theta\right\|_F^2\leq \frac{d^2r^2}{128T}$, $\text{rank}(\Theta)\leq r$, s.t. 
	\begin{align*}
		\mathbb{E}\left[R_T(\Theta)\right] = \Omega(dr\sqrt{T}).
	\end{align*}
\end{thm}

Above bound is tight when $r=d$ as it matches with the standard $d^2$-dimensional linear bandit lower bound, but for small $r$, our upper bound is larger than the lower bound by a factor of $\sqrt{d/r}$.

Nevertheless, we conjecture that $\Omega(d^{3/2}\sqrt{rT})$ is the correct lower bound for small $r$.
It is well-known that the regret lower bound for sparse linear bandit problem (dimension $d$, sparsity $s$) is $\Omega(\sqrt{sdT})$~\citep{lattimore2018bandit}.
Our low-rank linear bandit problem can be viewed as a $d^2$-dimensional linear bandit problem with $dr$ degrees of freedom in $\Theta^*$. Then, using the analogue of the degrees of freedom between sparse vectors and low-rank matrices, one can plug in $d^2$ for $d$ and $dr$ for $s$ in the sparse linear bandit regret lower bound and then achieve $\Omega(d^{3/2}\sqrt{rT})$ as our lower bound.

\section{Experiments}\label{sec:exp}
In this section, we compare the performance of OFUL and LowESTR to validate that it is crucial to utilize the low-rank structure. 

We run our simulation with $d_1 = d_2 = 10, r = 1$ and $d_1 = d_2 = 10, r = 3$.
In both settings, the true $\Theta^*\in\mathbb{R}^{d_1\times d_2}$ is a diagonal matrix. For $r = 1$, we set $\text{diag}(\Theta^*) = (0.5,0,\ldots,0)$ while for $r=3$, $\text{diag}(\Theta^*) = (0.5,0.5,0.5,0,\ldots,0)$.
For arms in both settings, we draw 256 vectors from $N(0,I_{d_1d_2})$ and standardize them by dividing their 2-norms, then we reshape all standardized $d_1d_2$-dimensional vectors to $d_1\times d_2$ matrices. 
We use these matrices as the arm set $\mathcal{X}$. 
For each arm $X\in\mathcal{X}$, the reward is generated by $y = \left\langle X,\Theta^*\right\rangle+\varepsilon$, where $\varepsilon\sim N(0,0.01^2)$.
We run both algorithms for $T = 3000$ rounds and repeat 100 times for each simulation setup to calculate the averaged regrets and their 1-standard deviation confidence intervals at every time step.

We leave the hyper-parameters of OFUL and LowESTR in the appendix (Section~\ref{sec:app_exp}). Regret comparison plots are displayed in Figure~\ref{fig:OFUL_ESTR}.

\begin{figure}[h]
	\centering
	\includegraphics[width=0.45\textwidth]{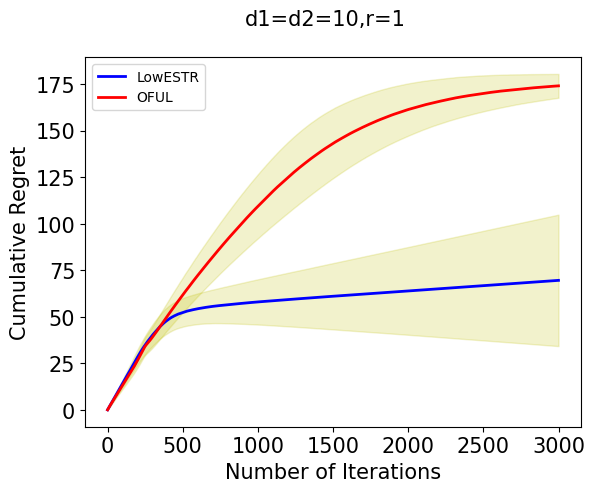}
	~
	\includegraphics[width=0.45\textwidth]{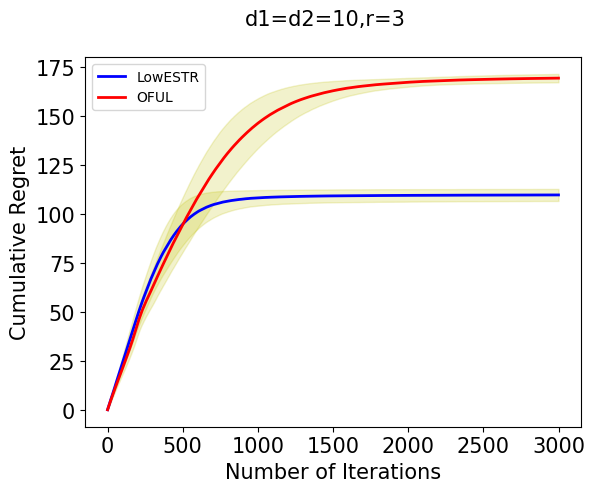}
	\caption{Regret comparison between OFUL and LowESTR for the two settings. We plot the averaged cumulative regret in red and blue curves, and 1-standard deviation for each method within the yellow shadow area.}
	\label{fig:OFUL_ESTR}
\end{figure}

We observe that in both plots, LowESTR incurs less regret comparing to OFUL within several hundreds of time steps. 
Further, as we increase the rank from $r=1$ to $r=3$, the cumulative regret gap between the two approaches becomes smaller. 
This phenomenon is compatible with our theory.

Other than the comparisons between OFUL and LowESTR, we also conduct simulations to see the sensitivity of LowESTR to the eigenvalue parameter $\omega_r$. We observe that LowESTR indeed performs better as $\omega_r$ goes larger, which again matches with our theory.
The detailed description and the plot for the sensitivity experiments are left to the appendix (Section~\ref{sec:app_exp}).

\section{Conclusion \& Future Work}\label{sec:conclusion}
In this paper, we studied the low-rank (generalized) linear bandit problem. 
We proposed LowLOC and LowGLOC algorithm for the linear and generalized linear setting, respectively.
Both of them enjoy $\tilde{O}((d_1+d_2)^{3/2}\sqrt{rT})$ regret. 
Further, our efficient algorithm LowESTR achieves $\tilde{O}((d_1+d_2)^{3/2}\sqrt{rT}/\omega_r)$ regret under mild conditions on the action set. 

There are several interesting directions we left for future work.
First, building on some preliminary ideas in Section~\ref{sec:ESTR} about how to extend LowESTR to the generalized linear setting, it should be possible to obtain a similar regret bound under certain regularity conditions on the link function.
Second, it will be interesting to investigate if one can design an efficient algorithm whose regret does not depend on $1/\omega_r$. 
Third, in Section~\ref{sec:lower}, we argued that $\tilde{O}((d_1+d_2)^{3/2}\sqrt{rT})$ should be a tight lower bound. 
It will be nice to formally prove this.

\section*{Acknowledgement}

AT acknowledges the support of NSF CAREER grant IIS-1452099 and an Adobe Data Science Research Award.
\bibliographystyle{apalike}
\bibliography{ref}
\newpage
\appendix
\section{Proof for Theorem~\ref{thm:LowLOC-EWAF}}\label{sec:app_LowLOC}
\begin{lemma}[Covering number for low-rank matrices, modified from~\citep{candes2011tight}] \label{lemma:cover}
	Let $S_r = \{\Theta\in\mathbb{R}^{d_1\times d_2}:\text{rank}(\Theta)\leq r, \left\|\Theta\right\|_F\leq 1\}$. Then there exists an $\epsilon-$net $\bar{S}_r$ for the Frobenius norm obeying
	\begin{align}
	|\bar{S}_r|\leq (9/\epsilon)^{(d_1+d_2+1)r}.
	\end{align}
\end{lemma}
\begin{proof}
	Use SVD decomposition: $\Theta = U\Sigma V^T$ of any $\Theta\in S_r$ obeying $\left\|\Sigma\right\|_F \leq 1$. We will construct an $\epsilon-$net for $S_r$ by covering the set of permissible $U,V$ and $\Sigma$.
	Let $D$ be the set of diagonal matrices with nonnegative diagonal entries and Frobenius norm less than or equal to one. We take $\bar{D}$ to be an $\epsilon/3$-net for $D$ with $|\bar{D}|\leq(9/\epsilon)^r$. Next, let $O_{d_1,r} = \{U\in\mathbb{R}^{d_1\times r}:U^TU=I\}$. To cover $O_{d_1,r}$, we use the $\left\|\cdot\right\|_{1,2}$ norm defined as
	\begin{align}
	\left\|U\right\|_{1,2} = \max_i \left\|U_i\right\|_{\ell_2},
	\end{align}
	where $U_i$ denotes the $i$th column of $\Theta$. Let $Q_{d_1,r} = \{U\in\mathbb{R}^{d_1\times r}:\left\|U\right\|_{1,2}\leq 1\}$. It is easy to see that $O_{d_1,r}\subset Q_{d_1,r}$ since the columns of an orthogonal matrix are unit normed. We see that there is an $\epsilon/3$-net $\bar{O}_{d_1,r}$ for $O_{d_1,r}$ obeying $|\bar{O}_{d_1,r}|\leq (9/\epsilon)^{d_1r}$. Similarly, let $P_{d_2,r} = \{V\in\mathbb{R}^{d_2\times r}:V^TV = I\}$. Define $R_{d_2,r} = \{V\in\mathbb{R}^{d_2\times r}:\left\|V\right\|_{1,2}\leq 1\}$, we have $P_{d_2,r}\subset R_{d_2,r}$. By the same argument, there is an $\epsilon/3$-net $\bar{P}_{d_2,r}$ for $P_{d_2,r}$ obeying $|\bar{P}_{d_2,r}|\leq (9/\epsilon)^{d_2r}$. We now let $\bar{S}_r = \{\bar{U}\bar{\Sigma}\bar{V}^T: \bar{U}\in O_{d_1,r}, \bar{V}\in P_{d_2,r}, \bar{\Sigma}\in\bar{D}\}$, and remark that $|\bar{S}_r|\leq |\bar{O}_{d_1,r}|^2|\bar{D}||\bar{P}_{d_2,r}|\leq (9/\epsilon)^{(d_1+d_2+1)r}$. It remains to show that for all $\Theta\in S_r$, there exists $\bar{\Theta}\in\bar{S}_r$ with $\left\|\Theta-\bar{\Theta}\right\|_F\leq \epsilon$.
	
	Fix $\Theta\in S_r$ and decompose it as $\Theta = U\Sigma V^T$. Then there exists $\bar{\Theta} = \bar{U}\bar{\Sigma}\bar{V}^T\in\bar{S}_r$ with $\bar{U}\in O_{d_1,r}$, $\bar{V}\in P_{d_2,r}$, $\bar{\Sigma}\in\bar{D}$ satisfying $\left\|U-\bar{U}\right\|_{1,2}\leq \epsilon/3$, $\left\|V-\bar{V}\right\|_{1,2}\leq \epsilon/3$ and $\left\|\Sigma-\bar{\Sigma}\right\|_F\leq \epsilon/3$. This gives
	\begin{align}
	\left\|\Theta-\bar{\Theta}\right\|_F &= \left\|U\Sigma V^T-\bar{U}\bar{\Sigma}\bar{V}^T\right\|_F\\
	&= \left\|U\Sigma V^T-\bar{U}\Sigma V^T+\bar{U}\Sigma V^T-\bar{U}\bar{\Sigma}V^T+\bar{U}\bar{\Sigma}V^T-\bar{U}\bar{\Sigma}\bar{V}^T\right\|_F\\
	&\leq \left\|(U-\bar{U})\Sigma V^T\right\|_F+\left\|\bar{U}(\Sigma-\bar{\Sigma})V^T\right\|_F+\left\|\bar{U}\bar{\Sigma}(V-\bar{V})^T\right\|_F.
	\end{align}
	For the first term, since $V$ is an orthogonal matrix, 
	\begin{align}
	\left\|(U-\bar{U})\Sigma V^T\right\|_F^2 &= \left\|(U-\bar{U})\Sigma\right\|_F^2\\
	&\leq \left\|\Sigma\right\|_F^2\left\|U-\bar{U}\right\|_{1,2}^2 \leq (\epsilon/3)^2.
	\end{align}
	Thus we have shown $\left\|(U-\bar{U})\Sigma V^T\right\|_F\leq \epsilon/3$, by the same argument, we also have $\left\|\bar{U}\bar{\Sigma}(V-\bar{V})^T\right\|_F\leq \epsilon/3$. For the second term, $\left\|\bar{U}(\Sigma-\bar{\Sigma})V^T\right\|_F = \left\|\Sigma-\bar{\Sigma}\right\|_F\leq\epsilon/3$. This completes the proof.
\end{proof}

\begin{lemma}[Online-to-Confidence-Set Conversion (adapted from Theorem 1 in~\citet{abbasi2012online})] \label{lemma:LOC}
	Suppose we feed $\{(X_s,y_s)\}_{s=1}^t$ into an online prediction algorithm which, for all $t\geq 0$, admits a regret $\sup_{\left\|\Theta\right\|_F\leq 1}\rho_t(\Theta)\leq B_t$. Let $\hat{y}_s$ be the prediction at time step $s$ by the online learner. Then, for any $\delta\in(0,0.25]$, with probability at least $1-\delta$, we have
	\begin{align}
	\mathbb{P}(\exists t\in\mathbb{N} \text{ such that } \Theta_*\notin\mathcal{C}_{t+1})\leq\delta,
	\end{align}
	where we define
	\begin{align}
	\beta_t(\delta) &= 1+2B_t+32\log\left(\frac{\sqrt{8}+\sqrt{1+B_t}}{\delta}\right)\\
	\mathcal{C}_{t+1} &= \{\Theta\in\mathbb{R}^{d_1\times d_2}:\left\|\Theta\right\|_F^2+\sum_{s=1}^{t}(\hat{y}_s-\langle\Theta,X_s\rangle)^2\leq 1+\beta_t(\delta)\}.
	\end{align}
\end{lemma}
\begin{lemma}[Regret of LowLOC Given Online Learner's Regret (adapted from Theorem 3 in~\citet{abbasi2012online})] \label{lemma:regret_LowLOC}
	Suppose $\sup_{\left\|\Theta\right\|_F\leq 1,\text{rank}(\Theta)\leq r} \rho_t(\Theta)\leq B_t$, where $\{B_t\}_{t=1}^T$ is a non-decreasing sequence. Then, for any $\delta\in(0,0.25]$, with probability at least $1-\delta$, for any $T\geq 0$, the regret of LowLOC algorithm is bounded as
	\begin{align}
	R_T = O\left(\sqrt{d_1d_2T(1+\beta_{T-1}(\delta))\log\left(1+\frac{T}{d_1d_2}\right)}\right),
	\end{align}
	where $\beta_t(\delta) = 1+2B_t+32\log\left(\frac{\sqrt{8}+\sqrt{1+B_t}}{\delta}\right)$.
\end{lemma}
\begin{lemma}[Theorem 3.2 in~\citep{cesa2006prediction}]\label{lemma:exp_cave}
	If the loss function $\ell(a,b)$ is exp-concave in its first argument for some $\eta>0$ (i.e. $F(a) = e^{-\eta\ell(a,b)}$ is concave for all $b$), then the regret of the exponentially weighted average forecaster in Equation~\ref{equ:exp_forecaster} (used with the same value of $\eta$) satisfies, for all $y_1,\ldots,y_n\in\mathcal{Y}$, we have $\Phi_\eta(\mathbf{R}_n)\leq \Phi_\eta(0)$.
\end{lemma}
\begin{lemma}[Proposition 3.1 in~\citep{cesa2006prediction}] \label{lemma:EWAF}
	If for some loss function $\ell$ and for some $\eta>0$, a forecaster satisfies $\Phi_\eta(\mathbf{R}_n)\leq \Phi_\eta(\mathbf{0})$ for all $y_1,\ldots,y_n\in\mathcal{Y}$, then the regret of the forecaster is bounded by 
	\begin{align}
	\widehat{L}_n - \min_{i=1,\ldots,N} L_{i,n}\leq \frac{\log(N)}{\eta}.
	\end{align}
\end{lemma}

\begin{proof} [Proof of Lemma~\ref{lemma:EW}]
	Let $y_t = \langle X_t, \Theta^*\rangle + \eta_t$. By subgaussian property, we have, for $0<\delta<1$,
	\begin{align}
	P\left(\max_{t=1,\ldots,T}|y_t|>1+\sqrt{2\log\left(\frac{2T}{\delta}\right)}\right)\leq \delta.
	\end{align}
	Let's denote above high probability event $\left\{\max_{t=1,\ldots,T}|y_t|\leq 1+\sqrt{2\log\left(\frac{2T}{\delta}\right)}\right\}$ by $G$, denote the online prediction at every round by $\hat{y}_t$.
	Define the $\varepsilon$-covering set for $S_r:=\{\Theta: \left\|\Theta\right\|_F\leq 1,\text{rank}(\Theta)\leq r\}$ by $\bar{S}_r$, which means, for any $\Theta\in S_r$, there exists a $\bar{\Theta}\in \bar{S}_r$, such that $\left\|\Theta-\bar{\Theta}\right\|_F\leq \varepsilon$. We prove that $|\bar{S}_r|\leq(9/\varepsilon)^{(d_1+d_2+1)r}$ in Lemma~\ref{lemma:cover}.
	
	One can easily show that $F(a) := e^{-\eta(a-b)^2}$ is concave in $a$ for all $|b|\leq 1+\sqrt{2\log\left(\frac{2T}{\delta}\right)}$ (this holds under event $G$) by choosing $	\eta = \frac{1}{2(2+\sqrt{2\log\left(\frac{2T}{\delta}\right)})^2}$, since $a$ refers to the prediction of exponential weighted average forcaster and thus we have $|a|\leq 1$ according to the construction. 
	So under event $G$, the squared loss $\ell$ is guaranteed to be exp-concave under above $\eta$ and Lemma~\ref{lemma:EWAF} can be applied here.

	We now bound the regret under event $G$. For an arbitrary $\Theta\in S_r$,
	\begin{align}
	\rho_T(\Theta) &= \mathbb\sum_{t=1}^{T}\left(\ell_t(\widehat{y}_t) - \ell_t(f_{\Theta,t})\right) \\
	&= \sum_{t=1}^{T}\left(\ell_t(\widehat{y}_t) - \ell_t(f_{\bar{\Theta},t}) +\ell_t(f_{\bar{\Theta},t})- \ell_t(f_{\Theta,t})\right) \ \text{where} \left\|\Theta-\bar{\Theta}\right\|_F\leq \epsilon, \ \bar{\Theta}\in\bar{S}_r\\
	&\leq \frac{\log|\bar{S}_r|}{\eta} + \sum_{t=1}^{T} \left(\ell_t(f_{\bar{\Theta},t})- \ell_t(f_{\Theta,t})\right) \ \text{by Lemma}~\ref{lemma:EWAF}\\
	& = \frac{\log|\bar{S}_r|}{\eta} + \sum_{t=1}^{T} \left((\langle\bar{\Theta},X_t\rangle-y_t)^2- (\langle \Theta,X_t\rangle-y_t)^2\right)\\
	&\leq \frac{\log|\bar{S}_r|}{\eta}+ \sum_{t=1}^{T}\left(2\left\|\Theta-\bar{\Theta}\right\|_F+2y_t\left\|\Theta-\bar{\Theta}\right\|_F\right)\\
	&\leq \frac{\log|\bar{S}_r|}{\eta}+2T\varepsilon+2T\varepsilon\sqrt{2\log\left(\frac{2T}{\delta}\right)}\\
	&= 2(d_1+d_2+1)r\log(\frac{9}{\varepsilon})\left(2+\sqrt{2\log\left(\frac{2T}{\delta}\right)}\right)^2+2T\varepsilon+2T\varepsilon\sqrt{2\log\left(\frac{2T}{\delta}\right)} \\
	&= O\left((d_1+d_2)r\log(T)\log\left(\frac{T}{\delta}\right)\right) \ \text{set }\varepsilon = 1/T.
	\end{align}	
	Above bounds hold for all $\Theta\in S_r$. This completes the proof.
\end{proof}

\begin{proof}[Proof of Theorem~\ref{thm:LowLOC-EWAF}]
	To obtain Theorem~\ref{thm:LowLOC-EWAF}, one just needs to plug Lemma~\ref{lemma:EW} into Lemma~\ref{lemma:regret_LowLOC}.
\end{proof}

\section{Proof for Theorem~\ref{thm:LowGLOC-EWAF}}\label{sec:app_LowGLOC}
\begin{lemma}[Online-to-Confidence-Set Conversion with NLL loss]\label{lemma:GLOC}
	Suppose we feed $\{(X_s,y_s)\}_{s=1}^t$ into an online prediction algorithm which, for all $t\geq 0$, admits a regret under negative log likelihood (NLL) loss $\sup_{\left\|\Theta\right\|_F\leq 1}\rho_t^{\textbf{GLB}}(\Theta)\leq B_t$. Let $\hat{y}_s$ be the prediction at time step $s$ by the online learner. Then, for any $\delta\in(0,0.25]$, with probability at least $1-\delta$, we have
	\begin{align}
	\mathbb{P}(\exists t\in\mathbb{N} \text{ such that } \Theta^*\notin \mathcal{C}_{t+1})\leq \delta,
	\end{align}
	where $C_t = \{\Theta\in\mathbb{R}^{d_1\times d_2}: \left\|\Theta\right\|_F+\sum_{s=1}^{t}\left(\hat{y}_s-\langle\Theta^*,X_s\rangle\right)^2\leq \beta_t^{\textbf{GLB}}(\delta)\}$ and $\beta_t^{\textbf{GLB}}(\delta) = 2+\frac{4}{\kappa_\mu}B_t+ \frac{32R^2}{\kappa_\mu^2}\log\left(\frac{R\sqrt{\frac{8}{\kappa_\mu^2}}+\sqrt{\frac{2}{\kappa_\mu}B_t+1}}{\delta}\right)$.
\end{lemma}
\begin{proof}
	According to the definition of $\rho_t^{\textbf{GLB}}(\cdot)$, we have
	\begin{align}
	B_t&\geq\rho_t^{\textbf{GLB}}(\Theta^*) \\
	&= \sum_{s=1}^{t}\ell_s(\hat{y}_s) - \ell_s(\langle\Theta^*,X_s\rangle)\\
	&\geq \sum_{s=1}^t (\hat{y}_s-\langle \Theta^*,X_s\rangle)\ell_s'(\langle\Theta^*,X_s\rangle)+\frac{\kappa_\mu}{2}(\hat{y}_s-\langle \Theta^*,X_s\rangle)^2 \tag{ Taylor expansion of $\ell_s$ at $\langle\Theta^*,X_s\rangle$ }\\
	&= \sum_{s=1}^t (\hat{y}_s-\langle \Theta^*,X_s\rangle)(-\eta_s)+\frac{\kappa_\mu}{2}(\hat{y}_s-\langle \Theta^*,X_s\rangle)^2.
	\end{align}
	Thus, rearranging the terms, we have
	\begin{align}
	\sum_{s=1}^{t}(\hat{y}_s-\langle \Theta^*,X_s\rangle)^2\leq \frac{2}{\kappa_\mu}B_t+\frac{2}{\kappa_\mu}\sum_{s=1}^{t}\eta_s(\hat{y}_s-\langle \Theta^*,X_s\rangle).
	\end{align}
	The remaining proof simply follows the proof of Lemma~\ref{lemma:LOC}. One can easily conclude that for any $\delta\in(0,0.25]$, with probability at least $1-\delta$
	\begin{align}
	\sum_{s=1}^{t}(\hat{y}_s-\langle\Theta^*,X_s\rangle)^2\leq 1+\frac{4}{\kappa_\mu}B_t+ \frac{32R^2}{\kappa_\mu^2}\log\left(\frac{R\sqrt{\frac{8}{\kappa_\mu^2}}+\sqrt{\frac{2}{\kappa_\mu}B_t+1}}{\delta}\right).
	\end{align}
	Adding $\left\|\Theta^*\right\|_F$ on both sides and using the fact that $\left\|\Theta^*\right\|_F\leq 1$, we complete the proof.
\end{proof}

\begin{lemma}[Regret of LowGLOC Given Online Learner's Regret]\label{lemma:regret_LowGLOC} 
	Suppose $\sup_{\left\|\Theta\right\|_F\leq 1} \rho_T^{\text{GLB}}(\Theta)\leq B_T^{\text{GLB}}$. Then, for any $\delta\in(0,0.25]$, with probability at least $1-\delta$, for any $T\geq 1$, the regret of LowGLOC algorithm is bounded by
	\begin{align}
	R_T = O\left(L\sqrt{\beta_{T-1}^{\text{GLB}}(\delta)Td_1d_2\log\left(1+\frac{T}{d_1d_2}\right)}\right),
	\end{align}
	where $\beta_t^{\textbf{GLB}}(\delta) = 2+\frac{4}{\kappa_\mu}B_t^{\text{GLB}}+ \frac{32R^2}{\kappa_\mu^2}\log\left(\frac{R\sqrt{\frac{8}{\kappa_\mu^2}}+\sqrt{\frac{2}{\kappa_\mu}B_t^{\text{GLB}}+1}}{\delta}\right) \forall t$.
\end{lemma}
\begin{proof}
	Define $V_{t-1} = I+\sum_{s=1}^{t-1}\text{vec}(X_s)^T\text{vec}(X_s)$ and 
	\begin{align}
	\widehat{\Theta}_t = \argmin_{\Theta\in\mathbb{R}^{d_1\times d_2}} \left(\left\|\Theta\right\|_F^2+\sum_{s=1}^{t-1}(\hat{y}_s-\langle \Theta,X_s\rangle)^2\right). 
	\end{align}
	One can express $C_{t-1}$ as
	\begin{align}
	\{\Theta\in\mathbb{R}^{d_1\times d_2}: \text{vec}(\Theta-\widehat{\Theta}_t)^TV_{t-1}\text{vec}(\Theta-\widehat{\Theta}_t)+\left\|\widehat{\Theta}_t\right\|_F^2+\sum_{s=1}^{t-1}(\hat{y}_s-\langle \Theta,X_s\rangle)^2\leq \beta_{t-1}(\delta)\}.
	\end{align}
	Thus, $C_{t-1}$ is contained in a bigger ellipsoid
	\begin{align}
	C_{t-1}\subseteq \{\Theta\in\mathbb{R}^{d_1\times d_2}: \text{vec}(\Theta-\widehat{\Theta}_t)^TV_{t-1}\text{vec}(\Theta-\widehat{\Theta}_t)\leq \beta_{t-1}(\delta)\}.
	\end{align}
	Now consider the regret at round $t$,
	\begin{align}
	\mu(\langle X^*,\Theta^*\rangle) - \mu(\langle X_t,\Theta^*\rangle) &\leq L_\mu|\left(\langle X^*,\Theta^*\rangle-\langle X_t,\Theta^*\rangle\right)| \\
	&\leq L_\mu\left(\langle X_t, \widetilde{\Theta}_t-\Theta^*\rangle\right)\\
	&\leq L_\mu|\langle X_t, \widetilde{\Theta}_t-\widehat{\Theta}\rangle|+L_\mu|\langle X_t, \widehat{\Theta}_t-\Theta^*\rangle|\\
	&\leq 2L_\mu\sqrt{\beta_{t-1}(\delta)}\left\|\text{vec}(X_t)\right\|_{V_{t-1}^{-1}} \text{   (Cauchy Schwartz) }.
	\end{align}
	Since the regret at every step cannot be bigger than $2L$,
	\begin{align}
	R_T &= \sum_{t=1}^{T} \mu(\langle X^*,\Theta^*\rangle) - \mu(\langle X_t,\Theta^*\rangle)\\
	&=\sum_{t=1}^{T}\min\left\{2L_\mu,2L_\mu\sqrt{\beta_{t-1}(\delta)}\left\|\text{vec}(X_t)\right\|_{V_{t-1}^{-1}}\right\}\\
	& = 2L_\mu\sqrt{\beta_{t-1}(\delta)}\sum_{t=1}^{T}\min\left\{\frac{1}{\beta_{t-1}(\delta)},\left\|\text{vec}(X_t)\right\|_{V_{t-1}^{-1}}\right\}\\
	&\leq 2L_\mu\sqrt{\beta_{t-1}(\delta)}\sqrt{T\sum_{t=1}^{T}\min\left\{\frac{1}{\beta_{t-1}(\delta)},\left\|\text{vec}(X_t)\right\|_{V_{t-1}^{-1}}\right\}}\\
	&\leq 2L_\mu\sqrt{\beta_{t-1}(\delta)}\sqrt{T\sum_{t=1}^{T}\min\left\{1,\left\|\text{vec}(X_t)\right\|_{V_{t-1}^{-1}}\right\}} \text{  ($\beta_{t-1}(\delta)$ is greater than 1)}\\
	&\leq 2L_\mu\sqrt{\beta_{t-1}(\delta)}\sqrt{2Td_1d_2\log\left(1+\frac{T}{d_1d_2}\right)}\\
	&= O\left(L_\mu\sqrt{\beta_{t-1}(\delta)Td_1d_2\log\left(1+\frac{T}{d_1d_2}\right)}\right).
	\end{align}
\end{proof}

\begin{lemma}[Regret of EW under NLL Loss]\label{lemma:EW_NLL}
    Let EW parameter $\eta:=\frac{\kappa_\mu}{\left(\sqrt{2R^2\log\left(\frac{2T}{\delta}\right)}+2c_\mu +2L_\mu\right)^2}$. 
	Then, for any $0<\delta<1$, with probability at least $1-\delta$, the regret of EW with expert predictions $f_{\Theta,t} = \langle \Theta,X_t\rangle$ under NLL loss satisfies
	\begin{align}
	B_T^{\text{GLB}} = \sup_{\left\|\Theta\right\|_F\leq 1,\text{rank}(\Theta)\leq r} \rho_T^{\textbf{GLB}}(\Theta) 
	& =O\left((d_1+d_2)r\log T \frac{\log\left(\frac{2T}{\delta}\right)L_\mu+c_\mu^2+L_\mu^2}{\kappa_\mu}\right)\\
	& = \widetilde{O}\left(\frac{L_\mu^2+c_\mu^2}{\kappa_\mu}(d_1+d_2)r\log\left(\frac{1}{\delta}\right)\right).
	\end{align}
\end{lemma}
\begin{proof}
	Under generalized linear bandit model, $y_t = \mu(\langle X_t,\Theta^*\rangle)+\eta_t$. By subgaussian property and $|\mu(\langle X_t,\Theta^*\rangle)|\leq |\mu(0)|+L_\mu|\langle X_t,\Theta^*\rangle|\leq c_\mu+L_\mu$, for $0<\delta<1$, we have
	\begin{align}
	P\left(\max_{t=1,\ldots,T}|y_t|>c_\mu+L_\mu+\sqrt{2R^2\log\left(\frac{2T}{\delta}\right)}\right) \leq \delta.
	\end{align}
	Again we denote above high probability event by $G$, denote the exponential weighted average forecaster at every round by $\hat{y}_t$. We use the same definition $S_r$ and $\bar{S}_r$ as last section. 
	
	We use Lemma~\ref{lemma:exp_cave} and Lemma~\ref{lemma:EWAF} to bound $\rho_T^{\textbf{GLB}}(\Theta)$. Then the first step is to find a proper $\eta>0$ such that $F(\hat{y}_t):=e^{\ell(\hat{y}_t,y_t)} = e^{-\eta m(\hat{y}_t)+\eta \hat{y}_ty_t}$ is concave. Taking derivatives we have,
	\begin{align}
	F''(\hat{y}_t) = \eta e^{-\eta m(\hat{y}_t)+\eta \hat{y}_ty_t}\left(\eta(y_t-\mu(\hat{y}_t))^2-\mu'(\hat{y}_t)\right).
	\end{align}
	Under event $G$, it's easy to show that 
	\begin{align}
	\frac{\mu'(\hat{y}_t)}{(y_t-\mu(\hat{y}_t))^2} \geq \frac{\kappa_\mu}{\left(\sqrt{2R^2\log\left(\frac{2T}{\delta}\right)}+2c_\mu +2L_\mu\right)^2},
	\end{align}
	since $|\mu(\hat{y}_t)|\leq |\mu(0)|+L|\hat{y}_t|\leq c_\mu+L_\mu$. Thus, taking $\eta:=\frac{\kappa_\mu}{\left(\sqrt{2R^2\log\left(\frac{2T}{\delta}\right)}+2c_\mu +2L_\mu\right)^2}$, $F(\cdot)$ is guaranteed to be concave with probability under event $G$.
	
	\begin{align}
	\rho_T^{\textbf{GLB}}(\Theta) &= \sum_{t=1}^{T}\left(\ell_t(\hat{y}_t)-\ell_t(\langle\Theta,X_t\rangle)\right)\\
	&\leq \sum_{t=1}^{T}\left(\ell_t(\hat{y}_t)-\ell_t(\langle\bar{\Theta},X_t\rangle)+\ell_t(\langle\bar{\Theta},X_t\rangle)-\ell_t(\langle\Theta,X_t\rangle)\right) \text{ where $\left\|\Theta-\bar{\Theta}\right\|_F\leq \varepsilon$ and $\bar{\Theta}\in\bar{S}_r$}\\
	&\leq \frac{\log|\bar{S}_r|}{\eta}+\sum_{t=1}^{T}\left(\ell_t(\langle\bar{\Theta},X_t\rangle)-\ell_t(\langle\Theta,X_t\rangle)\right)\\
	&\leq\frac{\log|\bar{S}_r|}{\eta}+\sum_{t=1}^{T}\langle \Theta-\bar{\Theta},X_t\rangle y_t+m(\langle\bar{\Theta},X_t\rangle)-m(\langle\Theta,X_t\rangle)\\
	&\leq \frac{\log|\bar{S}_r|}{\eta}+\sum_{t=1}^{T}|y_t|\left\|\Theta-\bar{\Theta}\right\|_F+|\langle\bar{\Theta}-\Theta,X_t\rangle(c_\mu+L_\mu)| \text{ (By Taylor expansion)}\\
	&\leq (d_1+d_2+1)r\log\left(\frac{9}{\varepsilon}\right)\frac{\left(\sqrt{2R^2\log\left(\frac{2T}{\delta}\right)}+2c_\mu +2L_\mu\right)^2}{\kappa_\mu}\\
	&+ T\left(2c_\mu+2L_\mu+\sqrt{2R^2\log\left(\frac{2T}{\delta}\right)}\right)\varepsilon\\
	&= O\left((d_1+d_2)r\log T \frac{\log\left(\frac{2T}{\delta}\right)L_\mu+c_\mu^2+L_\mu^2}{\kappa_\mu}\right),
	\end{align}
	where we take $\varepsilon = 1/T$.
\end{proof}

\begin{proof}[Proof for Theorem~\ref{thm:LowGLOC-EWAF}]
	One only needs to plug Lemma~\ref{lemma:EW_NLL} into Lemma~\ref{lemma:regret_LowGLOC}.
\end{proof}

\section{Proof for Theorem~\ref{thm:ESTR_regret}} \label{sec:app_ESTR}
The whole proof breaks down to two parts. 
Let $\Theta^* = U^*S^*V^{*T}$ be the SVD of $\Theta^*$. 
In the first part, we prove the convergence of estimated matrix $\widehat{\Theta}$ for $\Theta^*$, $\widehat{U}$ for $U^*$, and $\widehat{V}$ for $V^*$. In the second part, we plug the convergence result into the regret guarantee for LowOFUL in~\citet{jun2019bilinear} to achieve our final result.

\subsection{Analysis for Stage 1}
In order to analyze how the estimated subspaces are close to the true subspaces, we first present the definitions for sub-Gaussian matrix and restricted strong convexity (RSC) as below.

\begin{defini}[sub-Gaussian matrix (See \cite{wainwright2019high})]\label{def:subG}
	A random matrix $Z\in\mathbb{R}^{n\times p}$ is sub-Gaussian with parameters $(\Sigma,\sigma^2)$ if:
	\begin{itemize}
		\item each row $z_i^T\in\mathbb{R}^p$ is sampled independently from a zero-mean distribution with covariance $\Sigma$, and
		\item for any unit vector $u\in\mathbb{R}^p$, the random variable $u^Tz_i$ is sub-Gaussian with parameter at most $\sigma$.
	\end{itemize}
\end{defini}
\begin{defini}[Restricted Strong Convexity (RSC)~\citep{wainwright2019high}] \label{def:RSC}
	For a given norm $\left\|\cdot\right\|$, regularizer $\Phi(\cdot)$, and $X_1,\ldots,X_n\in\mathbb{R}^{d_1\times d_2}$, the matrix $\widehat{\Gamma} = \frac{1}{n}\widetilde{X}^T\widetilde{X}$, where $\tilde{x}_i:=\text{vec}(X_i)$ and $\widetilde{X}:=[\tilde{x}_1^T;\ldots;\tilde{x}_n^T]$, satisfies a \emph{restricted strong convexity} (RSC) condition with curvature $\kappa>0$ and tolerance $\tau_n^2$ if
	\begin{align}
	\widetilde{\Delta}^T\widehat{\Gamma}\widetilde{\Delta} = \frac{1}{n} \sum_{t=1}^{n}\langle X_t, \Delta\rangle^2 \geq \kappa\left\|\Delta\right\|^2-\tau_n^2\Phi^2(\Delta),
	\end{align}
	for all $\Delta\in\mathbb{R}^{d_1\times d_2}$, and we denote $\text{vec}(\Delta)$ by $\widetilde{\Delta}$.
\end{defini}

We prove the following theorem about distribution $D$ (see Assumption~\ref{assump:ESTR}) as below, see proof in Section~\ref{sec:app_RSC}.
\begin{thm}[Distribution $D$ satisfies RSC] \label{thm:RSC_subG}
	Sample $X_1,\ldots,X_n\in\mathbb{R}^{d_1\times d_2}$ from $\mathcal{X}$ according to $D$, and define $\tilde{x}_i:=\text{vec}(X_i)$, $\widetilde{X} = [\tilde{x}_1^T;\ldots;\tilde{x}_n^T]\in\mathbb{R}^{n\times d_1d_2}$ and $\widehat{\Gamma}:=\frac{1}{n}\widetilde{X}^T\widetilde{X}$. Then under Assumption~\ref{assump:ESTR}, there exists constants $c_1,c_2>0$, such that with probability $1-\delta$,
	\begin{align}
	\widetilde{\Theta}^T\widehat{\Gamma}\widetilde{\Theta} = \frac{1}{n}\sum_{i=1}^{n}\langle X_i, \Theta\rangle^2\geq \frac{c_1}{d_1d_2} \left\|\Theta\right\|_F^2-\frac{c_2(d_1+d_2)}{nd_1d_2}\left\|\Theta\right\|_{nuc}^2, \forall \Theta\in\mathbb{R}^{d_1\times d_2},
	\end{align}
	for $n = \Omega\left((d_1+d_2)\log\left(\frac{1}{\delta}\right)\right)$, where $\widetilde{\Theta}:=\text{vec}(\Theta)$.
\end{thm}

Theorem~\ref{thm:RSC_subG} states that sampling $X$ from $\mathcal{X}$ according to distribution $D$ guarantees that the sampled arms satisfies RSC condition. We further show that under RSC condition, the estimated $\widehat{\Theta}$ is guaranteed to converge to $\Theta$ at a fast rate in Theorem~\ref{thm:theta_conv}.

\begin{thm} \label{thm:theta_conv}
	Sample $X_1,\ldots,X_n\in\mathbb{R}^{d_1\times d_2}$ from $\mathcal{X}$ according to $D$. Then under Assumption~\ref{assump:ESTR}, any optimal solution to the nuclear norm optimization problem~\ref{equ:nuc} using $\lambda_n \asymp \frac{1}{n\min\{d_1,d_2\}}\log\left(\frac{n}{\delta}\right)\log\left(\frac{d_1+d_2}{\delta}\right)$ satisfies:
	\begin{align}
	\left\|\widehat{\Theta}-\Theta^*\right\|_F^2 
	\asymp \frac{(d_1+d_2)^3r}{n},
	\end{align}
	with probability $1-\delta$.
\end{thm}

The goal of stage 1 is to estimate the row/column subspaces of $\Theta^*$, below corollary characterizes their convergence. 
\begin{corollary}[adapted from~\cite{jun2019bilinear}]\label{cor:subspace}
	Suppose we compute $\widehat{\Theta}$ by solving the convex problem in Equation~\ref{equ:nuc} as an estimate of the matrix $\Theta^*$. After stage 1 of ESTR with $T_1 = \Omega\left(r(d_1+d_2)\right)$ satisfying the condition of Theorem~\ref{thm:theta_conv}, we have, with probability at least $1-\delta$,
	\begin{align}
	\left\|\widehat{U}_\perp^TU^*\right\|_F \left\|\widehat{V}_\perp^TV^*\right\|_F \leq \frac{\left\|\Theta^*-\widehat{\Theta}\right\|_F^2}{\omega_r^2}\leq C\frac{\lambda_{T_1}^2r}{\alpha_1^2\omega_r^2}:=\gamma(T_1) \asymp \frac{(d_1+d_2)^3r}{T_1\omega_r^2},
	\end{align}
	where $\omega_r>0$ denotes the lower bound of the $r$-th singular value of $\Theta^*$ and $C$ represents some constant.
\end{corollary}

\subsection{Analysis for Stage 2}
We present the useful lemmas proved in \citet{jun2017scalable} and combine them with our analysis of stage 1 to achieve the final result of Theorem~\ref{thm:ESTR_regret}.

\begin{lemma}[Corollary 1 in~\citep{jun2019bilinear}]\label{lemma:LowOFUL}
	The regret of LowOFUL with $\lambda_\perp = \frac{T}{k\log\left(1+\frac{T}{\lambda}\right)}$ is, with probability at least $1-\delta$,
	\begin{align}
	\widetilde{O}\left(\left(k+\sqrt{k\lambda}B+\sqrt{T}B_\perp\right)\sqrt{T}\right).
	\end{align}
\end{lemma}
\begin{lemma}[Modified from Theorem 5 in~\citep{jun2019bilinear}]\label{lemma:ESTR_sep}
	Suppose we run ESTR stage 1 with $T_1 = \Omega\left(r(d_1+d_2)\right)$. We invoke LowOFUL in stage 2 with $\lambda_\perp = \frac{T_2}{k\log(1+T_2/\lambda)}, B = 1, B_\perp = \gamma(T_1)$, the rotated arm sets $\mathcal{X'}_{\text{vec}}$ defined in LowESTR (Algorithm~\ref{algo:ESTR}). With probability $1-2\delta$, the regret of LowESTR is bounded by
	\begin{align}
	\widetilde{O}\left(T_1+T\cdot\frac{(d_1+d_2)^3r}{T_1\omega_r^2}\right).
	\end{align}
\end{lemma}
\begin{proof}
	Combining Lemma~\ref{lemma:LowOFUL} and definitions of parameters $B$, $B_\perp$, $\lambda$, $\lambda_\perp$ and $\gamma(T_1)$.
\end{proof}

\begin{proof} [Proof for Theorem~\ref{thm:ESTR_regret}]
	Suppose the assumptions in Lemma~\ref{lemma:ESTR_sep} hold. Setting $T_1 = \Theta\left((d_1+d_2)^{3/2}\sqrt{rT}\frac{1}{\omega_r}\right)$ in Lemma~\ref{lemma:ESTR_sep} leads to the regret
	\begin{align}
	\widetilde{O}\left((d_1+d_2)^{3/2}\sqrt{rT}\frac{1}{\omega_r}\right).
	\end{align}
\end{proof}
\section{Proof for Theorem~\ref{thm:RSC_subG}} \label{sec:app_RSC}
Throughout this proof, we use $\Sigma$ and $\sigma^2$ to denote the sub-Gaussian parameters defined in Definition~\ref{def:subG} for matrix $\widetilde{X}$ in the theorem.
\subsection{Useful Lemmas}
\begin{lemma}\label{lemma:conv_cl}
	For any constant $s\geq 1$, we have 
	\begin{align}
	\mathbb{B}_{nuc}(\sqrt{s})\cap \mathbb{B}_F(1) \subseteq 3\text{cl}\{\text{conv}\{\mathbb{B}_{rank}(s)\cap\mathbb{B}_F(1)\}\},
	\end{align}
	where the balls are taken in $\mathbb{R}^{d_1\times d_2}$, and $\text{cl}\{\cdot\}$ and $\text{conv}\{\cdot\}$ denote the topological closure and convex hull, respectively.
\end{lemma}
\begin{proof}
	Note that when $s>\min\{d_1,d_2\}$, the statement is trivial, since the right-hand set equals $\mathbb{R}_F(3)$, and the left-hand set is contained in $\mathbb{B}_F(1)$. Hence, we will assume $1\leq s\leq \min\{d_1,d_2\}$.
	
	Let $A,B\subseteq\mathbb{R}^{d_1\times d_2}$ be closed convex sets, with support function given by $\phi_A(z) = \sup_{\Theta\in A}\langle \Theta,z\rangle$ and $\phi_B$ similarly defined. It is well-known that $\phi_A(z)\leq \phi_B(z)$ if and only if $A\subseteq B$. We will now check this condition for the pair of sets $A = \mathbb{B}_{nuc}(\sqrt{s})\cap \mathbb{B}_F(1)$ and $B = 3\text{cl}\{\text{conv}\{\mathbb{B}_{rank}(s)\cap\mathbb{B}_F(1)\}\}$.
	
	For any $z\in\mathbb{R}^{d_1\times d_2}$, take $r:=\min\{d_1,d_2\}$, we have $z = U\Sigma V^T$ by SVD, where $U\in\mathbb{R}^{d_1\times r}$, $\Sigma\in\mathbb{R}^{r\times r}$, and $V\in\mathbb{R}^{d_2\times r}$. Let $S\subseteq \{1,\ldots,r\}$ be subset indexes for the top $\lfloor s \rfloor$ elements of $\text{diag}(\Sigma)$. We use $U_S$ and $V_S$ to denote submatrices of $U$ and $V$ with columns of indices in $S$ and use $\Sigma_S$ to denote the submatrix of $\Sigma$ with columns and rows of indices in $S$. Then we can write $z = U_S\Sigma_SV_S^T+U_S^\perp\Sigma_S^\perp V_S^{\perp T}$.
	
	Consider $\phi_A(z)$ below:
	\begin{align}
	\phi_A(z) &= \sup_{\Theta\in A} \langle \Theta, U_S\Sigma_SV_S^T+U_S^\perp\Sigma_S^\perp V_S^{\perp T}\rangle\\
	&\leq \sup_{\left\|U_SU_S^T\Theta\right\|_F\leq 1} \langle U_SU_S^T\Theta, U_S\Sigma_S V_S^T\rangle+\sup_{\left\|U_S^\perp U_S^{\perp T}\Theta\right\|_{nuc}\leq \sqrt{s}} \langle U_S^\perp U_S^{\perp T}\Theta, U_S^\perp\Sigma_S^\perp V_S^{\perp T}\rangle\\
	&\leq \left\|U_S\Sigma_SV_S^T\right\|_F+\sqrt{s}\left\|U_S^\perp\Sigma_S^\perp V_S^{\perp T}\right\|_{op} \ \text{by Holder inequality}\\
	&\leq \left\|U_S\Sigma_SV_S^T\right\|_F+\sqrt{s}\frac{1}{\lfloor s\rfloor}\left\|U_S\Sigma_SV_S^T\right\|_{nuc} \leq 3\left\|U_S\Sigma_SV_S^T\right\|_F.
	\end{align}
	Finally, note that $\phi_B(z) = \sup_{\Theta\in B}\langle \Theta, z\rangle = 3\max_{|S| = \lfloor s\rfloor}\sup_{\left\|U_SU_S^T\Theta\right\|_F\leq 1} \langle U_SU_S^T\Theta,U_S\Sigma_SV_S^T\rangle = 3\left\|U_S\Sigma_SV_S^T\right\|_F$, from which the claim follows.
\end{proof}

\begin{defini} \label{def:set}
	Define $\mathbb{K}(s):=\mathbb{B}_{rank}(s)\cap\mathbb{B}_F(1)$ and the cone set $\mathbb{C}(s):=\{v:\left\|v\right\|_{nuc}\leq\sqrt{s}\left\|v\right\|_F\}$, all matrices defined in these sets are in $\mathbb{R}^{d_1\times d_2}$.
\end{defini}
\begin{lemma} \label{lemma:dev_cond}
	For a fixed matrix $\Gamma\in\mathbb{R}^{d_1d_2\times d_1d_2}$, parameter $s\geq 1$, and tolerance $\delta>0$, suppose we have the deviation condition $(\tilde{v}:=\text{vec}(v))$
	\begin{align}
	|\tilde{v}^T\Gamma \tilde{v}|\leq \delta, \forall v\in\mathbb{K}(2s),
	\end{align}
	where $\mathbb{K}(2s)$ is defined in Definition~\ref{def:set}. Then
	\begin{align}
	|\tilde{v}^T\Gamma \tilde{v}|\leq 27\delta(\left\|v\right\|_F^2+\frac{1}{s}\left\|v\right\|_{nuc}^2), \forall v\in\mathbb{R}^{d_1\times d_2}. \label{inequ:dev_cond}
	\end{align}
\end{lemma}
\begin{proof}
	We begin by establishing the inequalities
	\begin{align}
	|\tilde{v}^T\Gamma\tilde{v}|&\leq 27\delta \left\|v\right\|_F^2, \forall v\in\mathbb{C}(s),\label{inequ:setC}\\
	|\tilde{v}^T\Gamma\tilde{v}|&\leq\frac{27\delta}{s}\left\|v\right\|_{nuc}^2, \forall v\notin \mathbb{C}(s) \label{inequ:setnotC},
	\end{align}
	where $\mathbb{C}(s)$ is defined in Definition~\ref{def:set}, the statement of this lemma then follows immediately.
	By rescaling, inequality~\ref{inequ:setC} follows if we can show that
	\begin{align}
	|\tilde{v}^T\Gamma\tilde{v}|&\leq 27\delta \ \text{for all $v$ such that $\left\|v\right\|_F = 1$ and $\left\|v\right\|_{nuc}\leq\sqrt{s}$} \label{inequ:setC_scale}.
	\end{align}
	By Lemma~\ref{lemma:conv_cl} and continuity, we further reduce the problem to proving the bound~\ref{inequ:setC_scale} for all vectors $v\in3\text{conv}\{\mathbb{K}(s)\} = \text{conv}\{\mathbb{B}_{rank}(s)\cap\mathbb{B}_F(3)\}$. Consider a weighted lienar combination of the form $v = \sum_i\alpha_iv_i$, with weights $\alpha_i\geq 0$ such that $\sum_i\alpha_i = 1$, and $rank(v_i)\leq s$ and $\left\|v_i\right\|_F\leq 3$ for each $i$. We can write
	\begin{align}
	\tilde{v}\Gamma \tilde{v} = \sum_{i,j}\alpha_i\alpha_j(\tilde{v}_i^T\Gamma\tilde{v}_j).
	\end{align}
	Applying inequality~\ref{inequ:dev_cond} to the vectors $v_i/3$, $v_j/3$ and $(v_i+v_j)/6$, we have
	\begin{align}
	|\tilde{v}_i^T\Gamma \tilde{v}_j| = \frac{1}{2}|(\tilde{v}_i+\tilde{v}_j)^T\Gamma(\tilde{v}_i+\tilde{v}_j)-\tilde{v}_i^T\Gamma \tilde{v}_i-\tilde{v}_j^T\Gamma \tilde{v}_j| \leq \frac{1}{2}(36+9+9)\delta = 27\delta
	\end{align}
	for all $i,j$, and hence $|\tilde{v}^T\Gamma \tilde{v}|\leq \sum_{i,j}\alpha_i\alpha_j(27\alpha) = 27\delta\left\|\alpha\right\|_2^2 = 27\delta$, establishing inequality~\ref{inequ:setC}.
	Now let's turn to inequality~\ref{inequ:setnotC}, note that $v\notin\mathbb{C}(s)$, we have
	\begin{align}
	\frac{|\tilde{v}^T\Gamma \tilde{v}|}{\left\|v\right\|_{nuc}^2} \leq \frac{1}{s}\sup_{\left\|u\right\|_{nuc}\leq\sqrt{s},\left\|u\right\|_F\leq 1} |u^T\Gamma u| \leq \frac{27\delta}{s},
	\end{align}
	where the first inequality follows by the substitution $u = \sqrt{s}\frac{v}{\left\|\right\|_{nuc}}$, the second follows by the same argument used for inequality~\ref{inequ:setC}. Rearrange above inequality, we establish inequality~\ref{inequ:setnotC}.
\end{proof}
\begin{lemma}[RSC condition] \label{lemma:RSC_cond}
	Suppose $s\geq 1$ and $\widehat{\Gamma}$ is an estimator of $\Sigma$ satisfying the deviation condition ($\tilde{v}:=\text{vec}(v)$)
	\begin{align}
	|\tilde{v}^T(\widehat{\Gamma}-\Sigma)\tilde{v}|\leq\frac{\lambda_{min}(\Sigma)}{54}, \forall v\in\mathbb{K}(2s),
	\end{align}
	where $\mathbb{K}(2s)$ is defined in Definition~\ref{def:set}. Then we have the RSC condition
	\begin{align}
	\tilde{v}^T\widehat{\Gamma}\tilde{v} \geq \frac{\lambda_{min}(\Sigma)}{2}\left\|v\right\|_F^2-\frac{\lambda_{min}(\Sigma)}{2s}\left\|v\right\|_{nuc}^2.
	\end{align}
\end{lemma}
\begin{proof}
	This result follows easily from Lemma~\ref{lemma:dev_cond}. Set $\Gamma = \widehat{\Gamma}-\Sigma$ and $\delta = \frac{\lambda_{min}(\Sigma)}{54}$, we have the bound
	\begin{align}
	|\tilde{v}^T(\widehat{\Gamma}-\Sigma)\tilde{v}|\leq\frac{\lambda_{min}(\Sigma)}{2}\left(\left\|v\right\|_F^2+\frac{1}{s}\left\|v\right\|_{nuc}^2\right).
	\end{align}
	Then
	\begin{align}
	\tilde{v}^T\widehat{\Gamma}\tilde{v}&\geq \tilde{v}^T\Sigma \tilde{v}-\frac{\lambda_{min}(\Sigma)}{2}\left(\left\|v\right\|_F^2+\frac{1}{s}\left\|v\right\|_{nuc}^2\right)\\
	&\geq \frac{\lambda_{min}(\Sigma)}{2}\left\|v\right\|_F^2-\frac{\lambda_{min}(\Sigma)}{2s}\left\|v\right\|_{nuc}^2,
	\end{align}
	where the last inequality follows from $\tilde{v}^T\Sigma \tilde{v}\geq \lambda_{min}(\Sigma)\left\|v\right\|_F^2$.
\end{proof}
\subsection{Proof for the Theorem~\ref{thm:RSC_subG}}
\begin{proof}
	Using the results in Lemma~\ref{lemma:RSC_cond}, together with the substitutions 
	\begin{align}
	\widehat{\Gamma}-\Sigma = \frac{1}{n}\widetilde{X}^T\widetilde{X}-\Sigma, \text{ and } s:=\frac{1}{c}\frac{n}{d_1+d_2}\min\{\frac{\lambda_{min}^2(\Sigma)}{\sigma^4},1\},
	\end{align}
	where $n\geq c(d_1+d_2) /\min\{\frac{\lambda_{min}^2(\Sigma)}{\sigma^4},1\}$ so $s\geq 1$, we see that it suffices to show that
	\begin{align}
	D(s):=\sup_{v\in\mathbb{K}(2s)}|\tilde{v}^T(\widehat{\Gamma}-\Sigma)\tilde{v}|\leq\frac{\lambda_{min}(\Sigma)}{54},
	\end{align}
	with high probability.
	
	Note that by modified Lemma 15 ( in Appendix $G$) in~\cite{loh2011high}, we simply change the $1/3$-covering set for sparsity vectors to the $1/3$-covering set for $\mathbb{K}(2s)$, whose covering number is $27^{2s(d_1+d_2+1)}$ by Lemma~\ref{lemma:cover}, and achieve
	\begin{align}
	\mathbb{P}(D(s)\geq t) \leq 2\exp\left(-c'n\min\left(\frac{t^2}{\sigma^4},\frac{t}{\sigma^2}\right)+2s(d_1+d_2+1)\log 27\right),
	\end{align}
	for some univeral constant $c'>0$. Setting $t = \frac{\lambda_{min}(\Sigma)}{54}$, we see that there exists some $c_2>0$, such that
	\begin{align}
	\mathbb{P}\left(D(s)\geq\frac{\lambda_{min}(\Sigma)}{54}\right) \leq 2\exp\left(-c_2n\min\left(\frac{\lambda_{min}^2(\Sigma)}{\sigma^4},1\right)\right),
	\end{align}
	which establishes the result.
	
	Set $\delta$ equals to the right side of last inequality, one can get the desired gurantee for $n$ in Theorem~\ref{thm:RSC_subG}.
\end{proof}

\section{Proof for Theorem~\ref{thm:theta_conv}}
\label{sec:app_theta_conv}
\subsection{Useful Lemmas}
\begin{lemma}[Converence under RSC, adapted from Proposition 10.1 in~\citet{wainwright2019high}] \label{lemma:RSC}
	Suppose the observations $X_1,\ldots,X_{n}$ satisfies the non-scaled RSC condition in Definition~\ref{def:RSC}, such that
	\begin{align}
	\frac{1}{T_1} \sum_{t=1}^{n}\langle X_t, \Theta\rangle^2 \geq \kappa\left\|\Theta\right\|_F^2 - \tau_{n}^2\left\|\Theta\right\|_{\text{nuc}}^2, \ \forall \Theta\in\mathbb{R}^{d_1\times d_2}.
	\end{align}
	Then under the event $G:=\{\left\|\frac{1}{n}\sum_{t=1}^{n}\eta_tX_t\right\|_{op}\leq \frac{\lambda_{n}}{2}\}$, any optimal solution $\widehat{\Theta}$ to Equation~\ref{equ:nuc} satisfies the bound below:
	\begin{align}
	\left\|\widehat{\Theta}-\Theta^*\right\|_F^2 \leq 4.5\frac{\lambda_{n}^2}{\kappa^2}r,
	\end{align}
	where $r = \text{rank}(\Theta^*)$ and $\frac{1}{\tau_{n}^2}\geq \frac{64r}{\kappa}$.
\end{lemma}
\subsection{Proof for Theorem~\ref{thm:theta_conv}}
\begin{proof}
	According to Theorem~\ref{thm:RSC_subG}, there exists constants $c_1$ and $c_2$ such that with probability at least $1-\delta$, we have below RSC condition
	\begin{align}
	\frac{1}{n}\sum_{t=1}^{n}\langle X_t,\Theta\rangle^2\geq \frac{c_1}{d_1d_2}\left\|\Theta\right\|_F^2-\frac{c_2(d_1+d_2)}{nd_1d_2}\left\|\Theta\right\|_{nuc}^2, \forall \Theta\in\mathbb{R}^{d_1\times d_2},
	\end{align}
	Lemma~\ref{lemma:RSC} can be applied under above RSC condition, then under event $G(\lambda_{n}):=\{\left\|\frac{1}{n}\sum_{t=1}^{n}\eta_tX_t\right\|_{op}\leq \frac{\lambda_{n}}{2}\}$, we can easily conclude the theorem.
	Thus, it remains to figure out $\lambda_{n}$ such that event $G(\lambda_n)$ can hold with high probability. 
	
	Define the rare event $E:=\left\{\max_{t=1,\ldots,T_1}|\eta_t|>\sqrt{2\log\left(\frac{4T_1}{\delta}\right)}\right\}$, so that $\mathbb{P}(E)\leq \delta/2$ can be proved by the definition of sub-Gaussian. By matrix Bernstein inequality, the probability of $G(\lambda_n)^c$ can be bounded by:
	\begin{align*}
	P\left(\left\|\frac{1}{n}\sum_{t=1}^{n}\eta_tX_t\right\|_{op}>\varepsilon\right) &\leq P\left(\left\|\frac{1}{n}\sum_{t=1}^{n}\eta_tX_t\right\|_{op}>\varepsilon \middle| E^c\right)+P(E)\\
	&\leq (d_1+d_2)\exp\left(\frac{-n\varepsilon^2/2}{ 2\log\left(\frac{4n}{\delta}\right)\max\{1/d_1,1/d_2\}+\varepsilon\sqrt{2\log\left(\frac{4n}{\delta}\right)}/3}\right)+\delta/2,
	\end{align*}
	where the last inequality is by matrix Bernstein using the fact that
	\begin{align}
	\max\left\{\left\|\sum_{t=1}^{n}\mathbb{E}\eta_t^2X_tX_t^T\right\|_{op},\left\|\sum_{t=1}^{n}\mathbb{E}\eta_t^2X_t^TX_t\right\|_{op}\right\}\leq 2n\log\left(\frac{4n}{\delta}\right)\max\{1/d_1,1/d_2\}.
	\end{align}
	
	For $(d_1+d_2)\exp\left(\frac{-n\varepsilon^2/2}{ 2\log\left(\frac{4n}{\delta}\right)\max\{1/d_1,1/d_2\}+\varepsilon\sqrt{2\log\left(\frac{4n}{\delta}\right)}/3}\right)\leq \delta/2$ to hold, we need 
	\begin{align}
	\epsilon^2 = \frac{C'}{n\min\{d_1,d_2\}}\log\left(\frac{n}{\delta}\right)\log\left(\frac{d_1+d_2}{\delta}\right),
	\end{align}
	holds for some constant $C'$. 
	Take $\lambda_{n} = 2\varepsilon$, we need $\lambda_{n}^2 = \frac{C}{n\min\{d_1,d_2\}}\log\left(\frac{n}{\delta}\right)\log\left(\frac{d_1+d_2}{\delta}\right)$ and under this condition we have $P(G(\lambda_{n}))\geq 1-\delta$. We complete the proof by noting that the scaling of the right hand side in Lemma~\ref{lemma:RSC} under above choice of $\lambda_n$ is indeed $\frac{(d_1+d_2)^3r}{n}$.
\end{proof}

\section{Proof for Theorem~\ref{thm:lower}}
\label{sec:app_lower}
\begin{proof}
	Take $\Delta = \sqrt{\frac{dr}{T}}\frac{1}{8\sqrt{3}}$, $\mathbb{\Theta} = \{\Theta = \begin{bmatrix}
	\theta_1^T\\
	\vdots\\
	\theta_r^T\\
	\mathbf{0}
	\end{bmatrix}\in\mathbb{R}^{d\times d}, \theta_i\in\{\pm\Delta\}^d, \forall i\in[r]\}$.
	For $i\in[r], j\in[d]$, define $\tau_{i,j} = T\wedge\min\{t:\sum_{s=1}^{t}X_{s,i,j}^2\geq \frac{T}{dr}\}$, where $X_{s,i,j}$ denotes the element on the $i$-th row and $j$-th column of matrix $X_s$. Then for a fixed $\Theta$, taking expectation over $X_t$, we have
	\begin{align}
		\mathbb{E}\left[R_T(\Theta)\right] &= \mathbb{E}_\Theta\sum_{t=1}^{T}\langle X^*-X_t,\Theta\rangle\\
		&=\Delta \mathbb{E}_\Theta\sum_{t=1}^{T}\sum_{i=1}^{r}\sum_{j=1}^{d}\left(\frac{1}{\sqrt{dr}}-X_{t,i,j}\text{sign}(\Theta_{i,j})\right)\\
		&\geq \frac{\Delta\sqrt{dr}}{2}\sum_{i=1}^{r}\sum_{j=1}^{d}\mathbb{E}_\Theta\left[\sum_{t=1}^{T}\left(\frac{1}{\sqrt{dr}}-X_{t,i,j}\text{sign}(\Theta_{i,j})\right)^2\right]\\
		&\geq \frac{\Delta\sqrt{dr}}{2}\sum_{i=1}^{r}\sum_{j=1}^{d}\mathbb{E}_\Theta\left[\sum_{t=1}^{\tau_{i,j}}\left(\frac{1}{\sqrt{dr}}-X_{t,i,j}\text{sign}(\Theta_{i,j})\right)^2\right].
	\end{align}
Define $U_{i,j}(x) = \sum_{t=1}^{\tau_{i,j}}\left(\frac{1}{\sqrt{dr}}-X_{t,i,j}x\right)^2$. Let $\Theta'\in\mathbb{\Theta}$ be another parameter matrix such that $\Theta' = \Theta$, except that $\Theta'_{i,j} = -\Theta_{i,j}$. Let $\mathbb{P}, \mathbb{P}'$ be the laws of $U_{i,j}$ with respect to the learner interaction measure induced by $\Theta$ and $\Theta'$. Then
\begin{align}
	\mathbb{E}_\Theta\left[U_{i,j}(1)\right]&\geq \mathbb{E}_{\Theta'}\left[U_{i,j}(1)\right]-(\frac{4T}{dr}+2)\sqrt{\frac{1}{2}D\left(\mathbb{P},\mathbb{P}'\right)}\\
	&\geq \mathbb{E}_{\Theta'}\left[U_{i,j}(1)\right]-\Delta(\frac{4T}{dr}+2)\sqrt{\mathbb{E}\left[\sum_{t=1}^{\tau_{i,j}}X_{t,i,j}^2\right]}\\
	&\geq \mathbb{E}_{\Theta'}\left[U_{i,j}(1)\right]-\Delta(\frac{4T}{dr}+2)\sqrt{\frac{T}{dr}+1}\\
	&\geq \mathbb{E}_{\Theta'}\left[U_{i,j}(1)\right]-\frac{8\sqrt{3}T\Delta}{dr}\sqrt{\frac{T}{dr}},
\end{align}
where in the first inequality we used Pinsker's inequality, the result in exercise 14.4 in~\citep{lattimore2018bandit}, the bound
\begin{align}
	U_{i,j}(1) = \sum_{t=1}^{\tau_{i,j}}\left(\frac{1}{\sqrt{dr}}-X_{t,i,j}\right)^2\leq 2\sum_{t=1}^{\tau_{i,j}}\frac{1}{dr}+2\sum_{t=1}^{\tau_{i,j}}X_{t,i,j}^2\leq \frac{2T}{dr}+2\left(\frac{T}{dr}+1\right) = \frac{4T}{dr}+2.
\end{align}
The second inequality in above follows from the chain rule for the relative entropy up to a stopping time in~\citep{lattimore2018bandit}:
\begin{align}
	D(\mathbb{P},\mathbb{P}')\leq \frac{1}{2}\mathbb{E}_\Theta\sum_{t=1}^{\tau_{i,j}}\langle X_{t},\Theta-\Theta'\rangle^2 = 2\Delta^2\mathbb{E}_\Theta\sum_{t=1}^{\tau_{i,j}}X_{t,i,j}^2.
\end{align}
The third inequality in above is true by the definition of $\tau_{i,j}$ and the fourth inequality holds by the assumption that $dr\leq 2T$.

Then,
\begin{align}
	\mathbb{E}_\Theta\left[U_{i,j}(1)\right]+\mathbb{E}_{\Theta'}\left[U_{i,j}(1)\right]&\geq \mathbb{E}_{\Theta'}\left[U_{i,j}(1)+U_{i,j}(-1)\right] - \frac{8\sqrt{3}T\Delta}{dr}\sqrt{\frac{T}{dr}}\\
	&= 2\mathbb{E}_{\Theta'}\left[\frac{\tau_{i,j}}{d}+\sum_{t=1}^{\tau_{i,j}}X_{t,i,j}^2\right]-\frac{8\sqrt{3}T\Delta}{dr}\sqrt{\frac{T}{dr}}\\
	&\geq \frac{2T}{d}-\frac{8\sqrt{3}T\Delta}{dr}\sqrt{\frac{T}{dr}} = \frac{T}{d}.
\end{align}
The proof is completed using an averaging number argument:
\begin{align}
	\sum_{\Theta\in\mathbb{\Theta}}R_T(\Theta)&\geq \frac{\Delta\sqrt{dr}}{2}\sum_{i=1}^{r}\sum_{j=1}^{d}\sum_{\Theta\in\mathbb{\Theta}}\mathbb{E}_\Theta\left[U_{i,j}(\text{sign}(\Theta_{i,j}))\right]\\
	&\geq \frac{\Delta\sqrt{dr}}{2}\sum_{i=1}^{r}\sum_{j=1}^{d}\sum_{\Theta_{-i,-j}}\sum_{\Theta_{i,j}\in\{\pm\Delta\}}\mathbb{E}_\Theta\left[U_{i,j}(\text{sign}(\Theta_{i,j}))\right]\\
	&\geq \frac{\Delta\sqrt{dr}}{2}\sum_{i=1}^{r}\sum_{j=1}^{d}\sum_{\Theta_{-i,-j}}\sum_{\Theta_{i,j}\in\{\pm\Delta\}}\frac{T}{dr} = 2^{dr-2}\Delta\sqrt{dr}T.
\end{align}
Hence there exists a $\Theta\in\mathbb{\Theta}$ such that $R_T(\mathcal{A},\Theta)\geq \frac{T\Delta\sqrt{dr}}{4} = \frac{dr\sqrt{T}}{32\sqrt{3}}$.
\end{proof}


\section{Preliminaries for EW}\label{sec:app_EW_prelim.tex}
We provide more information on the construction of standard \emph{exponentially weighted average forecaster}.
\paragraph{Prediction with Expert Advice.}
We use $\{f_{i,t}: i\in\mathcal{I}\}$ to denote the prediction of experts at round $t$, where $f_{i,t}$ is the prediction of expert $i$ at time $t$.
On the basis of the experts' predictions, the forecaster computes the prediction $\hat{y}_t$ for the next outcome $y_t$ and the true outcome $y_t$ is revealed afterwards. 
The regret of the learner relative to expert is defined by 
\begin{align*}
R_{i,T} = \sum_{t=1}^{T}\left(\ell_t(\widehat{p}_t)-\ell_t(f_{i,t})\right) = \widehat{L}_T-L_{i,T},
\end{align*}
where $L_{i,T} := \sum_{t=1}^{T}\ell_t(f_{i,t})$ and $\widehat{L}_T := \sum_{t=1}^{T}\ell_t(\widehat{p}_t)$. For linear prediction expert, we define $f_{\Theta,t}:=\langle \Theta,X_t\rangle$ and above reward matches with $\rho_T(\Theta)$.

\paragraph{Exponential Weighted Average Forecaster (EW).}
Suppose we have $N$ linear prediction experts. Define the regret vector at time $t$ as $\mathbf{r}_t = (R_{1,t},\ldots,R_{N,t})\in\mathbb{R}^N$ and the cumulative regret vector up to time $T$ as $\mathbf{R}_T = \sum_{t=1}^{T}\mathbf{r}_t$, then a weighted average forecaster is defined as
\begin{align*}
\widehat{p}_t = \sum_{i=1}^{N}\triangledown \Phi(\mathbf{R}_{t-1})_i f_{i,t}/ \sum_{j=1}^{N}\triangledown \Phi(\mathbf{R}_{t-1})_j
\end{align*}
where $\Phi(\cdot)$ denotes a potential function $\Phi: \mathbb{R}^N\rightarrow \mathbb{R}$ of the form $\Phi(\mathbf{u}) = \psi\left(\sum_{i=1}^{N}\phi(u_i)\right)$. $\phi:\mathbb{R}\rightarrow\mathbb{R}$ is any nonnegative, increasing and twice differentiable function, and $\psi: \mathbb{R}\rightarrow\mathbb{R}$ is any nonnegative, strictly increasing, concave and twice differentiable auxiliary function. 

\emph{Exponentially weighted average forecaster} is constructed using $\Phi_\eta(\mathbf{u}) = \frac{1}{\eta} \log\left(\sum_{i=1}^{N}e^{\eta u_i}\right)$, where $\eta$ is a positive parameter. The weights assigned to the experts are of the form: $\triangledown \Phi_\eta(\mathbf{R}_{t-1})_i = \frac{e^{\eta R_{i,t-1}}}{\sum_{j=1}^{N}e^{\eta R_{j,t-1}}}$. Thus, the \emph{exponentially weighted average forecaster} can be simplified to 
\begin{align*}
\hat{y}_t = \frac{\sum_{i=1}^{N}e^{-\eta L_{i,t-1}}f_{i,t}}{\sum_{j=1}^{N}e^{-\eta L_{j,t-1}}},
\end{align*}
as defined in the main text.


\section{More on Experiments}
\label{sec:app_exp}
\subsection{Parameter Setup for Comparing OFUL and LowESTR Simulation}
We present the parameter setups for the experiments in Section~\ref{sec:exp}.
\paragraph{OFUL:} failure rate: $\delta = 0.01$, horizon: $T = 3000$, standard deviation of the reward error $\sigma = 0.01$.

\paragraph{LowESTR:}

\begin{itemize}
	\item failure rate: $\delta = 0.01$.
	\item standard deviation of the reward error: $\sigma = 0.01$.
	\item least positive eigenvalue of $\Theta^*$: $\omega_r = 0.5$ for $r=1$ and $r=3$.
	\item horizon $T = 3000$, steps of stage 1: $T_1 = 200$, steps of stage 2: $T_2 = T-T_1$.
	\item penalization in Equation~\ref{equ:nuc}: $\lambda_{T_1} = 0.01\sqrt{\frac{1}{T_1}}$.
	\item gradient decent solving Equation~\ref{equ:nuc} step size: $0.01$.
	\item $k = r(d_1+d_2-r)$ in LowOFUL (Algorithm~\ref{algo:LowOFUL}).
	\item $B = 1, B_\perp = \sigma^2(d_1+d_2)^3r/(T_1\omega_r^2)$.
	\item $\lambda = 1, \lambda_\perp = \frac{T_2}{k\log\left(1+T_2/\lambda\right)}$.
\end{itemize}

\subsection{LowESTR: Sensitivity to $\omega_r$}
We prove a $\widetilde{O}\left((d_1+d_2)^{3/2}\sqrt{rT}\frac{1}{\omega_r}\right)$ regret for LowESTR algorithm in Section~\ref{sec:ESTR}. 
To complement this theoretical finding, we compare the performance of LowESTR on different values of $\omega_r\in\{0.05,0.1,0.2,0.3,0.4,0.5\}$. 

We run our simulation with $d_1=d_2=10,r=3$. The true $\Theta^*\in\mathbb{R}^{d_1\times d_2}$ is a diagonal matrix with $\text{diag} = (0.5,0.5,\omega_r,0,\ldots,0)$. 
The arm set is constructed in the same way as previous experiment and the reward is also generated by $y = \left\langle X,\Theta^*\right\rangle+\varepsilon$, where $\varepsilon\sim N(0,0.01^2)$. 
For each $\omega_r$ setting, we run LowESTR for 20 times to calculate the averaged regrets and their 1-sd confidence intervals.

Parameters for LowESTR are same as those of previous experiment except that $T_1 = \text{int}(100/\omega_r)$. 
The plot for cumulative regret at $T = 3000$ v.s. the value of $\omega_r$ is displayed in Figure~\ref{fig:omega_r comparison}.
We observe that as we increase the least positive singular value of $\Theta^*$: $\omega_r$, the cumulative regret up to $T=3000$ is indeed decreasing.

\begin{figure}[h]
    \centering
    \includegraphics[width=0.5\textwidth]{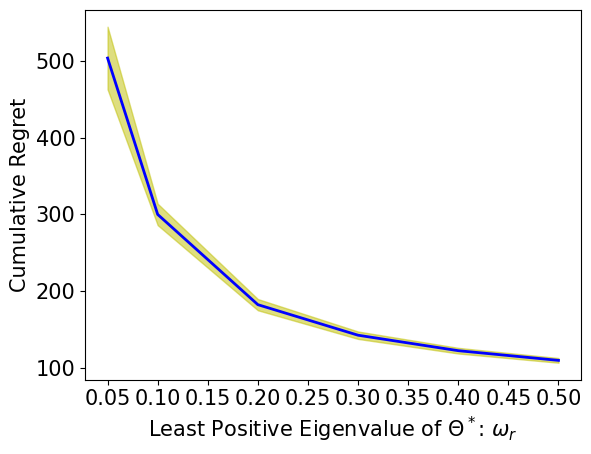}
    \caption{LowESTR: cumulative regret at $T=3000$ v.s. $\omega_r$. The yellow area represents the 1-standard deviation of the cumulative regret at $T=3000$.}
    \label{fig:omega_r comparison}
\end{figure}

\end{document}